\documentclass{article}
\usepackage[utf8]{inputenc}

\title{Approximability and Generalisation}
\author{Andrew J. Turner and 
Ata Kab\'an\\
School of Computer Science, University of Birmingham,\\ Edgbaston, Birmingham, B15 2TT, UK\\
 \{A.J.Turner, A.Kaban\}@bham.ac.uk}
\date{}

\usepackage[margin=2.54cm]{geometry}
\usepackage{amsmath, amsthm, mathrsfs}
\usepackage{amssymb}
\usepackage{mathtools}
\usepackage[colorlinks,citecolor=red,linkcolor=blue,]{hyperref}

\theoremstyle{definition}
\newtheorem{thm}{Theorem}[section]
\newtheorem{defn}[thm]{Definition}
\newtheorem{lem}[thm]{Lemma}
\newtheorem{cor}[thm]{Corollary}
\newtheorem{prop}[thm]{Proposition}

\newtheorem{ass}{Assumption}[]

\newcommand{\ApproxClass}{\Hil_A}
\newcommand{\X}{\mathcal{X}}
\newcommand{\Y}{\mathcal{Y}}
\newcommand{\R}{\mathbb{R}}
\newcommand{\N}{\mathbb{N}}

\newcommand{\Hil}{\mathcal{H}}
\newcommand{\A}{\mathcal{A}}
\newcommand{\Rad}{\widehat{\mathcal{R}}}
\newcommand{\RRad}{\mathcal{R}}

\newcommand{\distortion}[1][A]{\mathcal{D}_{#1}}
\newcommand{\distortionSam}[1][A]{\widehat{\mathcal{D}}_{#1}}
\newcommand{\B}{\mathcal{B}}

\DeclareMathOperator*{\E}{\mathbb{E}}
\DeclareMathOperator{\err}{err}
\DeclareMathOperator{\errsam}{\widehat{err}}
\DeclareMathOperator{\argmin}{argmin}
\DeclareMathOperator{\Var}{Var}
\newcommand{\x}{u} 

\begin{document}

\maketitle

\begin{abstract}
Approximate learning machines have become popular in the era of small devices, including quantised, factorised, hashed, or otherwise compressed predictors, and the quest to explain and guarantee good generalisation abilities for such methods has just begun. In this paper we study the role of approximability in learning, both in the full precision and the approximated settings of the predictor that is learned from the data, through a notion of sensitivity of predictors to the action of the approximation operator at hand. We prove upper bounds on the generalisation of such predictors, yielding the following main findings, for any PAC-learnable class and any given approximation operator.
1) We show that under mild conditions, approximable target concepts are learnable from a smaller labelled sample, provided sufficient unlabelled data. 2) We give algorithms that guarantee a good predictor whose approximation also enjoys the same generalisation guarantees.
3) We highlight natural examples of structure in the class of sensitivities, which reduce, and possibly even eliminate the otherwise abundant requirement of additional unlabelled data, and henceforth shed new light onto what makes one problem instance easier to learn than another. 
These results embed the scope of modern model compression approaches into the general goal of statistical learning theory, which in return suggests appropriate algorithms through minimising uniform bounds.\\

\end{abstract}

\section{Introduction}

The last decade has seen a tremendous increase of  interest in complex learning problems, such as deep neural networks, and learning in very high dimensional spaces, resulting in a
large number of parameters to be learned from the data. This is typically very resource-intensive in terms of memory, computation, and labelled training data; and consequently infeasible to deploy on devices with limited resources such as mobile phones, wearable devices, and the Internet of Things. 
Therefore, a plethora of model compression and approximation techniques have been proposed, such as quantisation, pruning, factorisation, random projection, hashing, and others \cite{Choudhary2020}. Rather intriguingly, many empirical findings on realistic benchmark problems seem to indicate that, despite a drastic compression of the complex model, such techniques often perform impressively well, with predictive accuracy comparable to that of full precision models. Below we mention just a few illustrative landmarks.

Quantisation of the weights of deep neural networks was proposed in BinaryConnect \cite{Courbariaux2015}, where a neural network with weights constrained to a single bit ($\pm 1$) was proposed and empirically demonstrated to achieve comparable results to a full precision network of the same size. These results were further refined and improved by the Quantised Neural Networks (QNN) training algorithm  \cite{Hubara2017}, and the idea was also extended to convolutional networks in Xnor-net \cite{Rastegari2016}.
Another compression scheme introduced in \cite{Han2015}, called Deep Compression, has employed a combination of pruning, quantisation, and Huffman coding to achieve similar results to the original network, with significant reduction in memory usage. 

Factorisation of the weights into low-rank matrices has been another common technique to reduce the size of a deep neural network (DNN), see \cite{Denil2013,Denton2014} for details.
Recent survey articles on a variety of model compression techniques specific to deep neural networks may be found in \cite{Choudhary2020,Cheng2020,Menghani2021}.

In a related work \cite{Ravi2019}, the authors propose to learn the high and low complexity networks simultaneously through a joint objective function that minimises not just their individual sample errors but also their disagreement. They found experimentally that this approach improves accuracy of both models, regardless of the model compression technique employed. 
While a theoretical explanation remains elusive, this was among the first attempts to move the goalpost from the restricted focus on the compressed model back to the fuller picture of the original model and consider these in tandem. 

Theoretical studies of model compression are much scarcer, and the interplay between model approximation and generalisation is not very well understood. 
Work taking an information theoretic approach \cite{Gao19} studied the trade-off between the compression granularity (rate) and the change it induces in the empirical error, using rate distortion theory. 
Follow-on work \cite{Bu21} extended their analysis to show that it is possible (on occasion) for compressed versions of pre-trained models to generalise even better than the original.

Another line of research exploited a notion of compression \cite{Arora2018,Zhou2019}. 
In \cite{Arora2018} a new compression framework was introduced for proving generalisation bounds, and their analysis indicated that resilience to noise implies a better generalisation for deep neural networks.
A PAC-Bayes bound was then proposed to give a non-vacuous generalisation bound on the compressed network in \cite{Zhou2019}.
This was further built upon in \cite{Baykal2018} and inspired a new algorithm along with a generalisation bound for the fully connected network.

In \cite{Suzuki2018}, compression-based bounds on a new pruning method for DNN was established, and 
more recently the authors also gave bounds for the full network \cite{Suzuki2020}. This latter work allows the compression-based bound to be converted into a bound for the full network, using the local Rademacher complexity of the Minkowski difference between the loss class of full networks and the loss class of compressed networks. This is therefore another instance, entirely complementary of the work of \cite{Ravi2019}, where the performance of the approximate model is linked back in some way to that of the full model, albeit a joint treatment has not been attempted.

In \cite{Ashbrock2020}, a stochastic Markov gradient decent was introduced to learn in memory limited setting directly in the discrete parameter space. They provide convergence analysis for their optimisation algorithm, but generalisation is only demonstrated experimentally on a handwritten digit recognition data set.

However, we conjecture a more fundamental connection between approximability and generalisation that is not specific to deep networks. Contrary to the increasingly sophisticated and specialised tools being developed for DNNs,
our aim here is to study this connection from first principles.
To do this, we want to ensure generalisation guarantees for learning with approximate models in general. 

On the other hand, we also hypothesise that target concepts which have low sensitivity to approximation may represent a benign trait of learning problems in general, which would imply easier learnability of the full precision model too. 
To substantiate this, we shall seek learning algorithms whose generalisation ability depends on the approximability of the target concept, irrespective of the form of the learned predictor being used in the full or approximated setting. 

\subsection{Contributions}
In the following roadmap we summarise the main contributions and findings of this paper:   

\begin{itemize}
    \item We define a notion of approximability of a predictor, which quantifies the average extent of sensitivity of its predictions when subjected to a given approximation operator (Section \ref{2.1}). This quantity will feature in our generalisation bounds.
    
    \item In Section \ref{2.2} we show that low sensitivity target functions require less labelled training data, provided we have access to a disjoint unlabelled set of sufficient size. 
    This sets the stage for approximability to be viewed as a benign trait for learning.
    
    \item In Section \ref{noBalcan} we develop a practical theory, showing that a constrained empirical risk minimisation algorithm with a modified loss function, which enforces approximability up to a given threshold, learns a predictor that is guaranteed to generalise well both in its full precision and its approximate forms (Proposition \ref{prop:generalisation bound knowing t}). Furthermore, we construct algorithms that implicitly optimise the trade-off managed by the sensitivity threshold (Proposition \ref{GuaranteeWithKnownsensitivityFunction}, and Corollary \ref{mainProp}). These algorithms can take advantage on additional unlabelled data without the requirement for it to be disjoint from the labelled set.
    
    \item For learning a good approximate predictor, we also give a variant of our algorithm that allows the user to control the above trade-off directly (Propositions \ref{prop:lambda} and \ref{prop:overline}). This may be useful in certain settings e.g. when low memory requirements prevail over prediction accuracy.
    
    \item Section \ref{diffClass} is devoted to studying our unlabelled data requirements.
    We show that, while the worst case unlabelled sample size requirement is necessarily large (Proposition \ref{prop:Crude2}), there are natural examples of structure whose presence may reduce, or may even eliminate the requirement for an unlabeled sample (Propositions \ref{prop:Union of ellipses centred at the origin} and \ref{clusters}). This analysis is largely independent of the hypothesis class employed. In addition, we also point out that structural restrictions on the hypothesis class itself can bring further insights -- in particular, for generalised linear models, weight sensitivity turns out to be sufficient for dimension-independent learning (Proposition \ref{prop:lin}).
\end{itemize}
Throughout the exposition of the main sections, we only consider deterministic approximation operators, keeping the reasoning and the formalism simple, and rooted in first principles. We discuss extensions in Section \ref{extensions}, including the use of stochastic approximation operators, and the possibility of obtaining faster rates.

\subsection{Closest related work}
We already highlighted two existing studies that considered both sides of model compression, namely the approximate predictor as well as the full predictor. Below we further discuss these in the light of our aims, approach, and findings, along with existing works that relate to ours in terms of either high-level ideas or technical aspects.

In a similar spirit to \cite{Ravi2019}, our inquiry concerns simultaneously both the approximate model and the full precision model.  
However, contrary to the empirical approach taken in \cite{Ravi2019}, where the heuristic nature of the algorithms make a theoretical understanding somewhat elusive, our approach is analytic. We employ Rademacher complexity analysis of the generalisation error as in \cite{Bartlett2002}, to give algorithm-independent uniform bounds on the generalisation for both approximate and approximable function classes. The uniform nature of these bounds justifies algorithms that minimise them. Therefore, our algorithms come with guarantees of good generalisation.
Our framework is general, and can be used to analyse the approximability and generalisation in tandem for any PAC-learnable machine learning problem. 

Our findings are consistent with those found in \cite{Suzuki2020}, with a difference in the approach that provides a different and more general angle.
Their focus is on translating already known bounds on compressed neural networks to the full uncompressed class. In contrast, we focus on showing that having good approximability (i.e. low sensitivity to approximation) improves generalisation bounds in PAC-learnable classes. In addition, we pursue a joint treatment of learning both the approximate and the full predictor simultaneously.

On a technical level, a key difference is that in \cite{Suzuki2020} the bounds depend on the local Rademacher complexity of the Minkowski difference of the loss classes of the full and the approximate predictors, which they are able to bound for some specific hypothesis classes; whereas, our bounds depend on the Rademacher complexity of the set of sensitivities of predictors from the hypothesis class. The Minkowski difference loses the coupling between the full and approximate predictor pairs which, in our approach is the key to taking advantage of structure in the set of sensitivities. These structures are not specific to the chosen hypothesis class, and instead uncover new general insights, as well as tighten the bounds effortlessly with elementary tools.

The works in \cite{Arora2018} and \cite{Zhou2019}, based on the idea of compression and resilience to noise, are also somewhat related to our work, on a high-level.
However, in both \cite{Arora2018} and \cite{Zhou2019} the generalisation bounds are for the compressed model only; whereas, our treatment provides both sides of the coin -- algorithms that learn a predictor that generalises both in its full precision and its approximate forms.
In \cite{Arora2018}, the focus is on bounding the classification error of the compressed predictor with the $\gamma$-margin loss (with $\gamma > 0$) of the full model for multi-class classification.
This corresponds to our general bounded, Lipschitz loss function.
Moreover, in \cite{Zhou2019} a PAC-Bayes approach is taken and so numerical tightness comes from data-dependent quantities in the bound that do not necessarily shed light on which structural traits of the problem are responsible for good generalisation.
In contrast, by employing Rademacher analysis we devote more time highlighting structural properties responsible for low complexity and good generalisation, so our approach and findings are complementary to these works. 

Our starting point in Section \ref{2.2} is the semi-supervised framework of \cite{Balcan2010}, where  
our approximability, or sensitivity of functions to approximation plays the role of an unlabeled error, and we replace VC entropy with Rademacher complexity to facilitate the use of our bounds outside the classification setting. However, from Section \ref{noBalcan} onward we depart from this framework in favour of simpler and more straightforwardly implementable bounds that fit our specific goals at the expense of a negligible additive term. In return, for our purposes the unlabelled data need not be disjoint from the labelled set, the sensitivity threshold can be optimised implicitly and automatically by our algorithm without appeal to structural risk minimisation, and in addition we study structural regularities that reduce or even eliminate the need of unlabelled data, which was not attempted in the previous work.

\section{Generalisation through approximability} \label{sec:2}
\subsection{Notations and preliminaries}\label{2.1}

Consider the input domain $\X \subset \R^d$, where $d$ denotes the dimensionality of the feature representation, and output domain $\Y \subseteq \R$.
Let $m \in \N$ and consider a sample $S \subset (\X \times \Y)^m$ of size $m$ drawn i.i.d. from an unknown distribution $D$. 
Let $\Hil$ be the hypothesis class; this is a set of functions from $\X$ to $\Y$. 
We consider a loss function $l \colon \Y \times \Y \to \R_+$.
Then we define the generalisation and empirical error of a function $f \in \Hil$ as
\[
    \err (f) \coloneqq \E_{(x,y) \sim D} [l(f(x),y)], \text{ and } \errsam (f) \coloneqq \frac{1}{m} \sum_{(x,y) \in S} l(f(x),y).
\]
The best function in the class will be denoted as $f^* \coloneqq \argmin_{f \in \Hil} \{ \err (f) \}$.

We let $\ApproxClass$ be the set of approximate functions from $\X$ to $\Y$.
Note $\ApproxClass$ need not be a subset of $\Hil$.
Then define an approximation operator $A \colon \Hil \to \ApproxClass$, which maps a hypothesis to their approximation.
Here $A$ is considered to be deterministic;  extensions to stochastic approximation schemes is discussed later in section \ref{extensions}.

\begin{defn}[Approximation sensitivity of a function]\label{def:approx}
Fix $p \in [1,\infty)$.
Then given a sample $S \subset \X^{m}$ of size $m$ drawn i.i.d. from the marginal distribution $D_x$, we define the true and empirical sensitivity as 
\[
    \distortion^p (f) \coloneqq \E_{x \sim D_x} [ \vert f(x) - Af (x) \vert^p ]^{\frac{1}{p}}, \text{ and } \distortionSam^p (f) \coloneqq \left( \frac{1}{m} \sum_{x \in S} \vert f(x) - Af (x) \vert^p \right)^{\frac{1}{p}}.
\]
\end{defn}

The choice of $p$-norm is left to the user in our forthcoming bounds. Formally, it would be sufficient to work with $p=1$, however often it may easier to specify a constraint in terms of the more familiar Euclidean norm ($p=2$) or some other member of the family of $p$-norms, and our results apply to any specification of $p$. 
More precisely,
by Jensen's inequality we have for all $p\geq 1$, that $\distortion^{1} (f) \leq \distortion^p (f)$ and $\distortionSam^1 (f) \leq \distortionSam^p (f)$, for all $f \in \Hil$. When the choice of $p$ is arbitrary, we may omit the upper index in our notation.

The approximating class $\Hil_A$ is typically chosen to be much smaller than the original class $\Hil$, implying a reduced complexity term in our generalisation bounds, at the expense of a larger empirical error $\errsam(Af)$ and the appearance of an additional sensitivity term $\distortion(f)$. We can think of $\Hil_A$ as a compressed model class whose elements occupy less memory, yet still expressive enough to represent the essence of $\Hil$. Examples include quantisation and other model compression schemes.
The granularity of approximation that we can afford is considered to be fixed. In memory-constrained settings this is constrained by the available hardware.

We now define sensitivity-restricted hypothesis classes
\[
    \Hil_t \coloneqq \{ f \in \Hil : \distortion (f) < t \} \text{ and } \widehat{\Hil}_t \coloneqq \{ f \in \Hil : \distortionSam (f) < t \}.
\]
We also define the class of sensitivities to be
\[
    \distortion \Hil \coloneqq \{ x \mapsto \vert   f(x) - Af (x) \vert   : f \in \Hil \}.
\]

We begin by stating the assumptions which we employ throughout the remainder of the paper. 
The first assumption is that the loss function is bounded and Lipschitz.
These allow us to invoke the theory of Rademacher complexity, as well as make the connection between the generalisation error and the sensitivity of a function.

\begin{ass} \label{ass:Bounded Lip loss}
    $l \colon \Y \times \Y \to \R_+$ is a bounded and $\rho$-Lipschitz loss function. 
    That is there exists $B>0$ such that 
    \[
        l (x,y) < B \text { and } \vert  l(x,y) - l(z,y)\vert   < \rho \vert   x - z \vert  , 
    \]
    for all $x,y,z \in \Y$. 
    By re-norming we may assume without loss of generality that $B=1$. 
\end{ass}

The second assumption is the uniform boundedness of the sensitivities. 
This lets us leverage the theory of Rademacher complexities for the class of sensitivities $\distortion (\Hil)$, which then allows us to shift the complexity terms from the full models to the approximate models. 

\begin{ass} \label{ass:bounded sensitivity}
    The set of sensitivities, $\distortion \Hil$, is uniformly bounded. 
    That is, there exists $C>0$ such that 
    \[
        \|   f - Af\|  _{\infty} < C, 
    \]
    for all $f \in \Hil$. 
\end{ass}

We start by giving a lemma that compares the true and empirical sensitivity.
This is where our estimates for the size of the unlabeled sample are derived.
We explore this topic further in Section \ref{diffClass}.

\begin{lem} \label{lem:unlabled error bound} 
With probability at least $1-\delta$ we have

\[
    \vert  \distortion^{1} (f) - \distortionSam^1 (f) \vert   \leq 2 \Rad_{S} (\distortion \Hil ) + 3 C \sqrt{\frac{\ln (\frac{2}{\delta})}{2 m}},
\]
for all $f \in \Hil$.
\end{lem}

\begin{proof}
By standard Rademacher bounds, it holds with probability at least $1-\delta$ that
\begin{align*}
     \vert  \distortion^{1} (f) - \distortionSam^1 (f) \vert  & = \left\vert   \E_{x \sim D_x} [ \vert   f(x) - Af (x) \vert   ] - \frac{1}{m} \sum_{x \in S} \vert   f(x) - Af (x) \vert   \right\vert  \\
     &\leq 2 \Rad_{S} (\distortion \Hil ) + 3 C \sqrt{\frac{\ln (\frac{2}{\delta})}{2 m}},
\end{align*}
as required. 
\end{proof}


We now relate the generalisation error of the full model with the generalisation error the approximate model through our notion of approximation sensitivity.
The following is a key lemma as it allows us to shift from the complexity of the full precision models to the low precision models.

\begin{lem} \label{lem:error to aproximate error}
Fix $t>0$.
We have the following bound
\[
    \vert  \err (f) - \err (A f) \vert   \leq \rho \distortion^{1} (f),
\]
for all $f \in \Hil_t$.
\end{lem}

\begin{proof}
Let $f \in \Hil_t$. Then, by Jensen's inequality and using the Lipschitz property of $l$ we have
\begin{align*}
    \vert  \err (f) - \err (A f)\vert   &\leq \E_{(x,y)\sim D} [\vert  l(f(x),y)- l(A f (x) , y)\vert  ]\\ & \leq \rho \E_{x \sim D_x} [ \vert   f(x) - Af (x) \vert   ] = \rho \distortion^{1} (f).
\end{align*}
This completes the proof.
\end{proof}

\subsection{Learning of low approximation sensitive predictors}\label{2.2}

Learning in high dimensional settings or complex model classes requires enormous training sets in general, or some fairly specific prior knowledge about the problem structure. However, many real-world problems possess benign traits that are hard to know in advance. Inspired by the practical success of approximate algorithms created by various model compression methods, in this section we investigate approximability as a potential benign trait for learning, by quantifying its effect on the generalisation error. More precisely, we elaborate on our intuition that, if a relatively complex target concept admits a simpler approximation which makes little alteration to its predictive behaviour, then it should be learnable from smaller training set of sizes.

The rationale is easy to see, as follows. Fix some approximation operator $A$ and associated sensitivity threshold $t\ge 0$. 
Then by standard Rademacher bounds, with probability at least $1-\delta$, for all $f\in \Hil_t$, we have
\begin{equation} \label{unif1}
    \err(f) \le \errsam(f)+2\rho\Rad_m(\Hil_t)+3\sqrt{\frac{\log(2/\delta)}{2m}}. 
\end{equation}
Let $f_t^*=\argmin_{f\in\Hil_t} \{ \err(f) \}$. To learn this function, we consider the ERM algorithm in the restricted class $\Hil_t$, that is we define the following algorithm
\begin{equation} \label{algo0}
    \hat{f}\coloneqq \argmin_{f\in\Hil_t} \{ \errsam(f) \}. 
\end{equation}
Then, by \eqref{unif1} and H\"offding's inequality we have, with probability at least $1-\delta$, that $\hat{f}$ satisfies
\begin{equation*} 
    \err(\hat{f})\le\err(f_t^*)+2\rho\Rad_m(\Hil_t)+4\sqrt{\frac{\log(3/\delta)}{2m}}.
\end{equation*}
Clearly, since $\Hil_t\subseteq \Hil$ then $\Rad_m(\Hil_t)\le \Rad_m(\Hil)$, and so whenever the concept we try to learn is actually in $\Hil_t$ (i.e. a low-sensitivity target function) then, depending on $t>0$, we have a tighter guarantee compared to that of an ERM algorithm in the larger class $\Hil$. 

Unfortunately, the algorithm in \eqref{algo0} is impractical because the specification of the function class $\Hil_t$ depends on the sensitivity function $\distortion$, which in turn depends on the true marginal distribution of the input data. It is often much easier to specify a larger function class $\Hil$ independent of the distribution, but this would ignore the sensitivity property and consequently lose out on the obtained tighter guarantee.

The first approach we consider is based on observing that the sensitivity function only depends on inputs and it is independent of the target values. Hence, we can make use of additional unlabelled data to estimate it, which is typically more widely available in applications. To this end, our first line of attack is similar in flavour with a classic semi-supervised framework proposed in \cite{Balcan2010}. This approach also allows us to use structural risk minimisation (SRM) to adapt the threshold parameter $t$. Therefore, balancing between the reduced complexity of the class and the potentially increased error of the best function on this class, yielding the following result.

\begin{prop}\label{BalcanMethod}
Fix an approximation operator $A$.
Suppose we have a disjoint i.i.d. unlabelled sample $S'_x\sim D_x^{m_u}$ of size $m_u$, and let $\epsilon_u>0$ s.t. $\sup_{f\in \Hil} \vert  \distortion(f)-\distortionSam(f)\vert  \le \epsilon_u$ with probability at least $1-\delta/2$ with respect to the random draw of $S'_x$.
Take an increasing sequence $(t_k)_{k\in\N} \subset \R_+$, and for each $k\in\N$ define $f_k^*
\coloneqq \argmin_{f \in \Hil_{t_k}} \{ \err (f) \}$. Let $w \colon \N \to \R$ such that $\sum_{k \in \N} w_k \leq 1$.
For each $f \in \Hil$ define 
$\hat{k}(f) \coloneqq \min \{ k \in \N : \distortionSam(f) \le t_{k}+\epsilon_u \}$.
Then, for all $k\in\N$ and all $f\in \hat{\Hil}_{t_k+\epsilon_u}$, with probability at least $1-\delta$, we have:
\begin{equation} \label{unif3}
    \err(f) \le 
\errsam(f)+2\rho\Rad_m(\hat{\Hil}_{t_{k}+\epsilon_u})+3\sqrt{\frac{\log(1/w_k)}{2m}} 
+3\sqrt{\frac{\log(4/\delta)}{2m}}.
\end{equation}
Now consider the following algorithm
\begin{equation} \label{algo3}
    \hat{f} \coloneqq \argmin_{f\in\Hil} \left\{ \errsam(f) + 2\rho\Rad_m \left(\hat{\Hil}_{t_{\hat{k}(f)} + \epsilon_u}\right) + 3\sqrt{\frac{\log\left(1/w_{\hat{k}(f)} \right)}{2m}} \right\}.
\end{equation}
Then, with probability at least $1-\delta$ we have 
\begin{equation} \label{guarantee3}
    \err(\hat{f}) \le \inf_{k\in\N}\left\{\err(f_{k}^*)+2\rho\Rad_m(\hat{\Hil}_{t_k+\epsilon_u})+3\sqrt{\frac{\log(1/w_{k})}{2m}}\right\}+4\sqrt{\frac{\log(6/\delta)}{2m}}.
\end{equation}
\end{prop}
With large enough $m_u$, i.e. sufficient unlabelled data, by Lemma \ref{lem:unlabled error bound}, with probability $1-\delta/2$ we have the magnitude of $\epsilon_u\le 2\Rad_{m_u}(\distortion\Hil)+3\sqrt{\frac{\log(4/\delta)}{2m_u}}$ 
can be made arbitrarily small -- a detailed account of this is discussed in Section \ref{diffClass}.
In addition, since the sensitivity function estimate only requires unlabelled data, which in many applications is easily available. 

Most importantly, since $\hat{\Hil}_{t_k+\epsilon_u}\subseteq \Hil$, we have $\Rad_m(\hat{\Hil}_{t_k+\epsilon_u})\leq \Rad_m(\Hil)$, so the gain obtained by restricting attention to approximable functions is apparent. The high probability guarantee \eqref{guarantee3} can also be made independent of the unlabelled data dependent classes by noting that {with the same probability 
$\hat{\Hil}_{t_k+\epsilon_u} \subseteq {\Hil}_{t_k+2\epsilon_u}$} (that is, if $\distortionSam^1 (f) \leq t + \varepsilon_u$ then with high probability $\distortion^1 (f) \leq t + 2 \varepsilon_u$ for all $f \in \hat{\Hil}_{t_k+\epsilon_u}$), so
$\Rad_m(\hat{\Hil}_{t_k+\epsilon_u})\leq \Rad_m({\Hil}_{t_k+2\epsilon_u})$; this follows from the definition of $\epsilon_u$ and the proof of \eqref{guarantee3}. 

The objective of the minimisation algorithm in \eqref{algo3} follows the idea of minimising the uniform bound \eqref{unif3}. It finds a good predictor along with the appropriate subclass of $\Hil$ to which it belongs. The sequence of sensitivity threshold candidates $(t_k)_{k\in\N}$, and the associated weights $(w_k)_{k\in\N}$, with $\sum_{k\in\N}w_k\le 1$ must be chosen before seeing any data (or instance, $w_k:=2^{-k}$), with $w_k$ representing an a priori belief in a particular $t_k$. 

\begin{proof}[Proof of Proposition \ref{BalcanMethod}]
For a fixed $t>0$,  
By the definition of $\epsilon_u$ then, with probability $1-\delta /2$, we have that  $f\in\Hil_t$ implies $f\in\hat{\Hil}_{t+\epsilon_u}$.
We shall pursue SRM by exploiting the disjoint unlabelled sample to define a nested sequence of function classes $\hat{\Hil}_{t_1+\epsilon_u}\subseteq \hat{\Hil}_{t_2+\epsilon_u} \dots \subseteq \hat{\Hil}_{t_k+\epsilon_u}\subseteq\hat{\Hil}_{t_{k+1}+\epsilon_u}\subseteq \dots$ where $k\in\N$. These classes depend on the unlabelled sample, but not on the labelled sample.
For any fixed $k\in\N$, the standard Rademacher bound 
implies with probability at least $1- (w_k\delta / 2)$ that all $f\in \hat{\Hil}_{t_k+\epsilon_u}$ satisfy
\[
    \err(f) \le \errsam(f)+2\rho\Rad_m(\hat{\Hil}_{t_k+\epsilon_u})+3\sqrt{\frac{\log(4/(\delta w_k))}{2m}}.
\]
Since $k \in \N$ is arbitrary, and $\sum_{k \in \N} w(k) \leq 1$, by taking a union bound it follows, with probability at least $1 - \frac{\delta}{2}$,
that uniformly for all $k\in\N$ and all $f\in \hat{\Hil}_{t_k+\epsilon_u}$ we have 
\[
    \err(f) \le \errsam(f)+2\rho\Rad_m(\hat{\Hil}_{t_k+\epsilon_u})+3\sqrt{\frac{\log(1/w_k)}{2m}}+3\sqrt{\frac{\log(4/\delta)}{2m}}.
\]
This proves \eqref{unif3}. 

To obtain \eqref{guarantee3} for $\hat{f}$ defined in \eqref{algo3}, we apply \eqref{unif3} to $\hat{f}$.  
By construction, 
$\hat{f}\in\hat{\Hil}_{t_{\hat{k}(\hat{f})}+\epsilon_u}$.
Recall also that with probability at least $1-\frac{\delta}{2}$ we have $f^*_{k}\in\Hil_{t_k}$, and so $f^*_{k}\in\hat{\Hil}_{t_k+\epsilon_u}$. 
Therefore, with probability at least $1-\delta/3$,
\begin{align}
\err(\hat{f})
    &\le 
    \errsam(\hat{f})+2\rho\Rad_m(\hat{\Hil}_{t_{\hat{k}(\hat{f})}+\epsilon_u})+3\sqrt{\frac{\log(1/w_{\hat{k}(\hat{f})})}{2m}}+3\sqrt{\frac{\log(2\cdot 3/\delta)}{2m}} \label{eq}\\
    &\le 
    \errsam(f^*_{k})+2\rho\Rad_m(\hat{\Hil}_{t_{k}+\epsilon_u})+3\sqrt{\frac{\log(1/w_{k})}{2m}}+3\sqrt{\frac{\log(6/\delta)}{2m}} \label{eq1},
\end{align}
for all $k\in\N$. In the last inequality we used the definition of $\hat{f}$ noting that the RHS of \eqref{eq} is minimised by $\hat{f}$, so any function replacing it will create an upper bound.
In addition, by H\"offding's inequality, we also have $\errsam(f^*_{k})\le \err(f^*_{k})+\sqrt{\frac{\log(6/\delta)}{2m}}$ with probability at least $1-\delta/6$. Combining with \eqref{eq1} and using the union bound, it follows with probability at least $1-\delta$ that 
\[
    \err(\hat{f}) \le 
    \err(f^*_{k})+2\rho\Rad_m(\hat{\Hil}_{t_{k}+\epsilon_u})+3\sqrt{\frac{\log(1/w_{k})}{2m}}+4\sqrt{\frac{\log(6/\delta)}{2m}},
\]
for all $k\in\N$. Finally, choosing $k$ to minimise the bound concludes the proof. 
\end{proof}

\subsection{A joint approach to sensitivity and generalisation}
\label{noBalcan}
The conceptually straightforward approach of the previous subsection implies that any target concept that is robust to the effects of approximating it by a low-complexity predictor, will require less labelled examples to be learned; and, a regularised ERM algorithm can accomplish this learning task. The algorithm adaptively trims the original function class to the relevant subset of low-sensitivity predictors, and consequently returns a low-sensitivity element of an otherwise potentially much larger function class. 

The appeal of this finding lies not only to serve as a possible explanation towards the question of what makes some instances of a learning problem easier than others. Also, by the low-sensitivity property, such predictor should be usable in its approximated form in memory-constained settings. Indeed, for any $t>0$, if $\hat{f}\in\Hil_t$, then by Lemma \ref{lem:error to aproximate error} we have
\begin{align}
    \err(A\hat{f}) = \err(\hat{f})+ \left(\err(A\hat{f})-\err(\hat{f})\right)
    \le \err(\hat{f})+\rho\distortion(\hat{f})\le \err(\hat{f})+\rho t. \label{additiveTerm}
\end{align}
In other words, for a predictor with low-sensitivity, using $A\hat{f}$ instead of $\hat{f}$ will only incur a small additive error of up to $\rho t$.

In this section we are interested in a more practical formulation of this tandem. The approach presented so far, beyond its conceptual elegance, has some practical drawbacks:
(1) it requires an additional disjoint unlabelled data set, and
(2) it requires computing the Rademacher complexity of the restricted class, which in itself is a hard optimisation problem.

To get around these limitations, we shall take a different approach, by modifying the loss function to explicitly encode the fact that we are interested in a good low-complexity approximate predictor.
More precisely, for a given threshold value $t>0$, we start by defining the following constrained ERM algorithm to learn the approximate target $f_t^*$:
\begin{equation} \label{algo1}
    \hat{f}_t \coloneqq \argmin_{f \in \Hil_t} \{ \errsam (A f) \} 
\end{equation}
One can also adapt $t$ and estimate $\distortion(f)$ in the same way using SRM as in the previous section -- we omit doing these, and we will then observe shortly some simple tweaks that make these steps unnecessary. 

The following result shows that the function returned by algorithm \eqref{algo1} achieves two different functionalities simultaneously, as it not only produces a good approximate predictor with quantified error guarantee including the price to pay for the approximation, but the $\hat{f}$ itself is a good predictor whenever the problem admits an appoximable target function.
\begin{prop}
\label{prop:generalisation bound knowing t}
Fix an approximation operator $A$ and $t\geq 0$. Define $f_t^* \coloneqq \argmin_{f \in \Hil_t} \{ \err (f) \}$ and $g_t^* \coloneqq \argmin_{g \in A \Hil_t} \{ \err (g) \}$. 
Then with probability at least $1-\delta$, the function, $\hat{f}_t$, returned by the algorithm in \eqref{algo1} satisfies all of the following: 
\begin{align}
    \err(A \hat{f}_t) & \leq \min \{ \err (A f_t^*), \err (g_t^*) \} + 2 \rho \Rad_S (\ApproxClass) + 4 \sqrt{\frac{\ln (\frac{9}{\delta})}{2m}}, \label{Af+AgBound} \\
    \err (A \hat{f}_t) & \leq \err (f_t^*) + \rho t + 2 \rho \Rad_S (\ApproxClass) + 4 \sqrt{\frac{\ln (\frac{9}{\delta})}{2m}} \label{AfBound} \\
    \err (\hat{f}_t) & \leq \err (f_t^*) + 2 \rho t + 2 \rho \Rad_S (\ApproxClass) + 4 \sqrt{\frac{\ln (\frac{9}{\delta})}{2m}} \label{fBound},
\end{align}
simultaneously.
\end{prop}
We note that $\rho t \to 0$ as $t \to 0$; however, as $t$ decreases the choice of predictors in $\Hil_t$ decreases and so $\err (f_t^*)$ would be expected to increase.
That is, the choice of $t$ balances the trade-off between the sensitivity term, $\rho t$, and the error term, $\err (f_t^*)$.

Proposition \ref{prop:generalisation bound knowing t} allows us to view learning and compression as two sides of the same coin. Eq. \eqref{fBound} tells us that a low-sensitivity target function is more easily learnable and a \emph{constrained} ERM algorithm learns it up to a constant factor of its sensitivity. Indeed, suppose $f^*=f_t^*$ i.e. the target function has sensitivity below $t$. Then the error of our constrained ERM is guaranteed to be much smaller than the worst case error of learning $f^*$ in the whole class $\Hil$.
At the same time, \eqref{AfBound} provides a guarantee for the approximate predictor $A\hat{f}$ that can potentially be deployed in low-memory settings, and here the additive term proportional to $t$ is the price of model compression.
Remarkably, both of these two seemingly different goals are accomplished by the same function returned by the learning algorithm \eqref{algo1}.
Moreover the algorithm \eqref{algo1} also performs very well compared to other predictors in the space of approximations as the bound \eqref{Af+AgBound} gives guarantees for $A \hat{f}_t$ when compared to both $A f_t^*$ (the approximation of best predictor in $\Hil$ with sensitivity lat most $t$) and $g_t$ (the best approximate predictor with some predictor in $\Hil$ that is compressed to $g$ and has sensitivity at most $t$).
That is, \eqref{Af+AgBound} suggests that $A \hat{f}_t$ is guaranteed to be no worse than learning in the approximated class.

\begin{proof}[Proof of Proposition \ref{prop:generalisation bound knowing t}] 
By Rademacher bounds and Talagrand's contraction lemma, we have with probability at least $1-\frac{2\delta}{9}$, that
\begin{align}
    \err(A \hat{f}_t) &\leq \errsam (A \hat{f}_t) + 2 \Rad_S (l \circ A \Hil_t) + 3 \sqrt{\frac{\ln (\frac{2 \cdot 9}{2\delta})}{2m}}\nonumber\\& \leq \errsam (A \hat{f}_t) + 2 \rho \Rad_S (A \Hil_t) + 3 \sqrt{\frac{\ln (\frac{9}{\delta})}{2m}}.\label{e1}
\end{align}
By definition of $\hat{f}_t$ we have $\errsam (A \hat{f}_t) \leq \min \{ \errsam (A f_t^*), \errsam (g_t^*) \}$.
Using this together with H\"offding's inequality, we have with probability $1-\frac{\delta}{9}$ that both
\begin{align*}
    \errsam (A \hat{f}_t) &\leq \errsam (A f_t^*) \leq \err (A f_t^*) + \sqrt{\frac{\ln (\frac{9}{\delta})}{2m}}, \text{ and } \\
    \errsam (A \hat{f}_t) &\leq \errsam (g_t^*) \leq \err (g_t^*) + \sqrt{\frac{\ln (\frac{9}{\delta})}{2m}},
\end{align*}
hold separately.
Therefore, by the union bound and the fact that $\Rad_S (A \Hil_t) \leq \Rad_S (\ApproxClass)$ we have with probability at least $1 -\frac{4\delta}{9}$, that \eqref{Af+AgBound} holds.
Similarly, as $\errsam (A \hat{f}_t) \leq \errsam (A f_t^*)$ and by Lemma \ref{lem:error to aproximate error}, we have with probability at least $1 -\frac{\delta}{9}$, that
\begin{align}
    \errsam (A \hat{f}_t) \leq \errsam (A f_t^*) &\leq \err (A f_t^*) + \sqrt{\frac{\ln (\frac{9}{\delta})}{2m}}
    \leq \err (f_t^*) + \rho \distortion^{1} (f_t^*) + \sqrt{\frac{\ln (\frac{9}{\delta})}{2m}}\\ 
    &\leq \err (f_t^*) + \rho t + \sqrt{\frac{\ln (\frac{9}{\delta})}{2m}}.\label{e2}
\end{align}
Combining the above three inequalities and the fact that $\Rad_S (A \Hil_t) \leq \Rad_S (\ApproxClass)$ we have with probability at least $1 -\frac{2\delta}{9}$, that
\begin{align}
    \err (A \hat{f}_t) \leq \err (f_t^*) + 2 \rho t + 2 \rho \Rad_S (\ApproxClass) + 4 \sqrt{\frac{\ln (\frac{9}{\delta})}{2m}}. \label{e3}
\end{align}
This proves \eqref{AfBound}.
The second part follows 
by using Lemma \ref{lem:error to aproximate error}, Jensen's inequality and $\hat{f}_t \in \Hil_t$, so we have 
\[
    \err (\hat{f}_t) \leq \rho \distortion^{1} (\hat{f}_t) + \err (A \hat{f}_t) \leq \rho \distortion (\hat{f}_t) + \err (A \hat{f}_t) \leq \rho t + \err (A \hat{f}_t).
\]  
Taking the union bound for each of the equations completes the proof.
\end{proof}

%

Next we show that in this formulation we can relax the fixed parameter $t$ constraining the function class, without the use of SRM. To avoid clutter, here we suppose the functional form of $f\rightarrow \distortion(f)$ is know -- again, this can be estimated from a disjoint unlabelled data set as in the previous section. In addition, for the case when we only care about the the approximate function $A\hat{f}$, the next subsection will also provide an alternative that does not necessitate additional unlabelled data. 

To this end, consider the following algorithm. 
\begin{equation}
\hat{f} \coloneqq \argmin_{f\in \Hil} \{ \errsam (A f) + \rho \distortion (f) \} \label{AlgoWithKnownsensitivity}
\end{equation}
Here the first term is our modified loss function as before, and the second term acts as a regulariser that implicitly constrains the function class. The following result shows that $\hat{f}$ returned by this algorithm behaves as the previous algorithm \eqref{algo1},
while it also automatically adapts the class-constraining sensitivity threshold $t$.

\begin{prop} \label{GuaranteeWithKnownsensitivityFunction}
Fix an approximation operator $A$. For $t>0$, let $f_t^*:=\argmin_{f\in\Hil_t}\{ \err(f) \}$.
For the function $\hat{f}$ returned by the learning algorithm given in \eqref{AlgoWithKnownsensitivity}, with probability at least $1-\delta$ we have both of the following 
\begin{align}
    \err(A \hat{f})&\leq  \inf_{t>0}\left\{\err (f^*_t) + 2 \rho t\right\} 
    + 2 \rho \Rad_S (\ApproxClass) + 4 \sqrt{\frac{\ln (\frac{8}{\delta})}{2m}}, \text{~and} \label{part_2} \\
    \err (\hat{f})& \leq \inf_{t>0}\left\{\err (f^*_t) + 2 \rho t\right\} 
    + 2 \rho \Rad_S (\ApproxClass) + 4 \sqrt{\frac{\ln (\frac{8}{\delta})}{2m}}, \label{part_1}
\end{align}
simultaneously.
\end{prop}

\begin{proof}[Proof of Proposition \ref{GuaranteeWithKnownsensitivityFunction}]
Using Lemma \ref{lem:error to aproximate error}, and standard Rademacher bounds, we have with probability at least $1-\frac{\delta}{4}$, that 
\begin{align}
    \err (\hat{f}) &\leq \err (A \hat{f}) + \rho \distortion^{1} (\hat{f}) \label{l0} \\ 
    &\leq \errsam (A \hat{f}) + \rho\distortion^{1} (\hat{f}) + 2 \Rad_S (l \circ \ApproxClass) + 3 \sqrt{\frac{\ln (\frac{2 \cdot 4}{\delta})}{2m}}. \label{l1}
\end{align}
Let $f^*:=\argmin_{t>0} \{ \err (f_t^*) + 2\rho\distortion^{1} (f_t^*) \}$.
Then by the definition of $\hat{f}$ and the Hoeffding bound we obtain with a probability of at least $1- \frac{\delta}{8}$, that 
\begin{equation}
    \errsam (A \hat{f}) + \rho\distortion^{1} (\hat{f}) \leq
    \errsam(A f^*)+\rho\distortion^{1} (f^*)
    \leq \err(A f^*)+\rho\distortion^{1} (f^*)
    +\sqrt{\frac{\ln (\frac{8}{\delta})}{2m}}. \label{l2}
\end{equation}
Then, by Lemma \ref{lem:error to aproximate error} definition of $f^*$ we have 
\[
    \err(A f^*)+\rho\distortion^{1} (f^*) \le \err(f^*) + 2\rho\distortion^{1} (f^*) \le \err(f_t^*)+2\rho\distortion^{1} (f_t^*), 
\]
for all $t>0$.
Hence, $\err(A f^*)+\rho\distortion^{1} (f^*) \le \inf_{t>0}\{\err(A f_t^*) + 2\rho\distortion^{1} (f_t^*) \}$, and substituting into \eqref{l2} yields
\[
    \errsam (A \hat{f}) + \rho\distortion^{1} (\hat{f}) \leq \inf_{t>0} \{\err(A f^*_t)+ 2\rho\distortion^{1} (f^*_t)\}+ \sqrt{\frac{\ln (\frac{8}{\delta})}{2m}}.
\]
with probability at least $1-\frac{\delta}{8}$.
By the Talagrand contraction lemma we have $\Rad_S (l \circ \ApproxClass) \leq \rho \Rad_S (\ApproxClass)$, and so combining with \eqref{l1} and then by a union bound we have with probability at least $1-\frac{\delta}{2}$ that 
\[
    \err (\hat{f}) \leq \inf_{t>0} \{\err(A f^*_t)+ 2\rho\distortion^{1} (f^*_t)\} + \sqrt{\frac{\ln (\frac{8}{\delta})}{2m}}
    + 2 \rho \Rad_S (\ApproxClass) + 4 \sqrt{\frac{\ln (\frac{8}{\delta})}{2m}}.
\]
Noting that $\distortion(f_t^*)\leq t$ completes the proof of \eqref{part_1}. 
Eq. \eqref{part_2} also follows, with probability at least $1-\frac{\delta}{2}$, since $\err(A\hat{f})$ is upper bounded by the RHS of \eqref{l0}, by adding the non-negative term $\rho \distortion (f)$.
\end{proof}

From Proposition \ref{GuaranteeWithKnownsensitivityFunction} we see again that, for any fixed approximation function $A$ such that $\Hil_A$ has smaller complexity than $\Hil$, if the target function has a low sensitivity (i.e. $\distortion(f^*)$ is small), then it is learnable from fewer labels than an arbitrary target from $\Hil$ would be. Of course, there may be learning problems where $f^*$ has low error but high sensitivity for the pre-defined $A$, but the algorithm in \eqref{AlgoWithKnownsensitivity} finds a function that automatically balances between generalisation error and sensitivity.

It is straightforward to use an estimate of $\distortion(f)$ in the algorithm, as the following corollary to Proposition \ref{GuaranteeWithKnownsensitivityFunction}. 
\begin{cor} \label{mainProp}
Fix an approximation operator $A$, and consider the following algorithm.
\begin{align}
    \hat{f} \coloneqq \argmin_{f\in\Hil}\{\errsam(Af)+\rho\distortionSam(f)\}.\label{algoSmart}
\end{align}
Let $\epsilon_u$ such that $\sup_{f\in \Hil} \vert  \distortion(f)-\distortionSam(f)\vert  \le \epsilon_u$ with probability at least $1-\frac{\delta}{2}$ with respect to $D_x^{m_u}$ where $m_u\ge m$ 
For $t>0$, let $f_t^*:=\argmin_{f\in\Hil_t}\{ \err(f) \}$.
Then with probability at least $1-\delta$, the function $\hat{f}$ satisfies both
\begin{align*}
    \err(A \hat{f})&\leq  \min_{t>0}\left\{\err (f^*_t) + 2 \rho t\right\} 
    + 2 \rho \Rad_S (\ApproxClass) + (4+\rho) \sqrt{\frac{\ln (\frac{16}{\delta})}{2m}}+\rho\epsilon_u, \text{~and} \\
    \err (\hat{f})& \leq \min_{t>0}\left\{\err (f^*_t) + 2 \rho t\right\} 
    + 2 \rho \Rad_S (\ApproxClass) + (4+\rho) \sqrt{\frac{\ln (\frac{16}{\delta})}{2m}}+\rho\epsilon_u,
\end{align*}
simultaneously.
\end{cor}

\begin{proof}
This follows by the same steps as the proof of Proposition \ref{GuaranteeWithKnownsensitivityFunction} combined with Lemma \ref{lem:error to aproximate error}.
\end{proof}

With sufficient unlabelled data $\epsilon_u$ can be made arbitrarily small. Moreover, unlike the approach in the previous section where the function class depends on the unlabelled data through the sensitivity estimate, here the implicit adaptation of $t$ enables reusing the labelled points in the estimating of sensitivity. The advantage of the algorithm analysed in Proposition \ref{BalcanMethod} is statistical consistency, since given enough labelled data the generalisation error converges to that of the best predictor of the class. However, if the goal is to obtain an approximate predictor, we pay the price of an additive sensitivity term \eqref{additiveTerm}, and Corollary \ref{mainProp} shows that allowing such term enables a much more efficient implementation without sacrificing the essence of the theoretical guarantee on generalisation.


\subsection{Managing the trade-off between sample error and sensitivity for the approximate predictor }\label{ImplementableAlgo}
The analyses from Proposition \ref{GuaranteeWithKnownsensitivityFunction} and Corollary \ref{mainProp} have shown that the associated algorithms have an implicit ability to realise the optimal trade-off between the sample error of $A\hat{f}$ and the sensitivity term, $t$, without any effort or tuning parameter from the user.

However, there may be situations when a different trade-off may be wanted and in such a case it may be desirable to manage this trade-off as a tuning parameter. This is especially relevant for practical applications in memory-constrained settings, where obtaining a good approximate predictor $A\hat{f}$ is the sole interest. 
For instance, we may only care about very low sensitivity function at the expense of a slightly raised error, or vice-versa. Or we might like to explore multiple trade-offs as in a bi-objective approach. Another instance of this is when unlabelled data is also scarce but an analytic upper bound can be derived on the sensitivity function up to an unknown constant. 

Conceptually, a good way to address this sort of issues would be to take back control over the threshold parameter $t$ using the learning algorithm in \eqref{algo1} (with or without estimating the sensitivity). However, the constrained optimisation formulation can be awkward to perform in practice. Below we suggest a more user-friendly form of the algorithm, and show that its solution is close to that of \eqref{algo1}.

For each $\lambda \geq 0$ consider the following algorithm
\begin{align}
    \tilde{f}_{\lambda} &\coloneqq \argmin_{f \in \Hil} \{ \errsam (A f) + \lambda \distortionSam (f) \}.\label{algo2}
\end{align}
Algorithms of this form, including the exploitation of unlabelled data in the regularisation term, have been in use in practice for a long time \cite{Chapelle}, see also \cite{semisupSurvey}. While the authors in \cite{Balcan2010} point out that this is not theoretically justified in general, we are able to justify it within our approximability objective. The regularisation parameter $\lambda$ balances the two terms of the objective function, and in addition to potential availability of prior knowledge, there is a wide range of well-established model selection methods available to set this parameter in practice. 

To this end, we shall compare the error of $\tilde{f}_{\lambda}$ from algorithm \eqref{algo2} with that of $\hat{f}_t$ from the algorithm given in \eqref{algo1}.
The following proposition shows that, for any specification of $\lambda$, there is a value of $t>0$ such that the errors of these two predictors are close, up to additive terms that decay with the sample size.

\begin{prop}\label{prop:lambda}
Let $\epsilon_u > 0$ be such that $\sup_{f\in \Hil} \vert  \distortion(f)-\distortionSam(f)\vert  \le \epsilon_u$ with probability at least $1-\delta/4$ with respect to $D_x^{m_u}$ where $m_u\ge m$.
For any $\lambda>0$, there exists $t > 0$ such that with probability at least $1-\delta$ we have
\begin{align}
    \err (A \tilde{f}_{\lambda}) - \err (A \hat{f}_t) & \leq 4 \rho \Rad_S (\ApproxClass) + 6 \sqrt{\frac{\ln (\frac{8}{\delta})}{2m}} + 2\lambda\epsilon_u 
    \label{eqn:True Equiv of implementations}.
\end{align}
\end{prop}


\begin{proof}[Proof of Proposition \ref{prop:lambda}]
Take 
$    t \leq \distortion (\tilde{f}_{\lambda})$. Then from the definition of algorithm \eqref{algo1} we have 
$\distortion (\hat{f}_t) \leq t \leq \distortion (\tilde{f}_{\lambda})$.
Using this, the definition of $\tilde{f}_{\lambda}$, and Lemma \ref{lem:unlabled error bound}, it follows with probability at least $1-\frac{\delta}{4}$ that
\begin{align*}
    \errsam (A \tilde{f}_{\lambda}) + \lambda \distortionSam (\tilde{f}_{\lambda}) & \leq \errsam (A \hat{f}_t) + \lambda \distortionSam (\hat{f}_t) \\
    & \leq \errsam (A \hat{f}_t) + \lambda \distortion (\hat{f}_t) + \lambda \epsilon_u 
    \\
    & \leq \errsam (A \hat{f}_t) + \lambda \distortion (\tilde{f}_{\lambda}) + \lambda \epsilon_u. 
\end{align*}
Rearranging, and using Lemma \ref{lem:unlabled error bound} again, we have with probability at least $1-\delta/2$ that
\[
    \errsam (A \tilde{f}_{\lambda}) - \errsam (A \hat{f}_t) \leq \lambda (\distortion (\tilde{f}_{\lambda}) - \distortionSam (\tilde{f}_{\lambda})) + \lambda\epsilon_u 
    \leq 2\lambda\epsilon_u. 
    \label{eqn:Sample Equiv of implementations}
\]
This shows that the sample errors of the two predictors are close. 

Now, to prove \eqref{eqn:True Equiv of implementations} we use standard Rademacher bounds with probability at least $1 - \delta/2$ on $\ApproxClass$ twice, combined with \eqref{eqn:Sample Equiv of implementations} and the union bound, we have with probability $1-\delta$, 
\begin{align*}
    \err (A \tilde{f}_{\lambda}) - \err (A \hat{f}_t) & = (\err (A \tilde{f}_{\lambda}) - \errsam (A \tilde{f}_{\lambda})) + (\errsam (A \tilde{f}_{\lambda})) - \errsam (A \hat{f}_t))\\&\hspace{3cm} + {(\errsam (A \hat{f}_t) - \err (A \hat{f}_t))} \\
    & \leq 2\lambda\epsilon_u 
    + {2} \left(2 \rho \Rad_S (\ApproxClass) + 3 \sqrt{\frac{\ln (\frac{8}{\delta})}{2m}}\right),
\end{align*}
as required.
\end{proof}

Finally, we now address the case when instead of estimating the sensitivity from unlabelled data we have an analytic upper bound on this function, in the case of some specific choice of function class and approximation operator, up to some unknown absolute constant. The constant will be subsumed into the tuning parameter $\lambda$.

Let $\overline{\distortion}(\cdot)$ be a mapping from $\Hil$ to $\R_+$ where there exists $c>0$ such that for all $f \in \Hil$, we have $\distortion(f)\le c\cdot\overline{\distortion}(f)$.
Note that, $\overline{\distortion}(\cdot)$ does not depend on the sample.
Now, for each $\lambda \geq 0$ define the following algorithm
\begin{equation} \label{algo2a}
    \overline{f}_{\lambda} \coloneqq \argmin_{f \in \Hil} \{ \errsam (A f) + \lambda \overline{\distortion}(f) \}.
\end{equation}

Furthermore, let $\hat{f}_t$ be the predictor returned by algorithm \eqref{algo1}, and $\overline{f}_t$ the predictor from a version of the same algorithm that uses $\overline{\distortion}(\cdot)$ in place of the unknown $\distortion(\cdot)$. Then $\overline{f}_t$ will have a guarantee of the same form as before in Proposition \ref{prop:generalisation bound knowing t} where $t$ is now a threshold on $\overline{\distortion}(\cdot)$ rather than ${\distortion}(\cdot)$.
The following proposition shows that the error of  $\overline{f}_{\lambda}$ is close to that of $\overline{f}_t$.

\begin{prop}\label{prop:overline}
For any $\lambda>0$, there exists $t > 0$ such that with probability at least $1-\delta$ we have
\begin{equation}
    \err(A\overline{f}_{\lambda})-\err(A\overline{f}_{t}) \leq 
    4 \rho \Rad_S (\ApproxClass) + 6 \sqrt{\frac{\ln (\frac{8}{\delta})}{2m}}.    
\end{equation}
\end{prop}

\begin{proof}
Take $t>0$ such that $\overline{\distortion}(\overline{f}_t) \leq t \leq\overline{\distortion}(\overline{f}_{\lambda})$. 
Consequently, by the definition of $\overline{f}_{\lambda}$, we have
\begin{equation*}
    \errsam(A\overline{f}_{\lambda})+\lambda\overline{\distortion}(\overline{f}_{\lambda}) \leq \errsam(A\overline{f}_{t})+\lambda\overline{\distortion}(\overline{f}_{t}) \leq \errsam(A\overline{f}_{t})+\lambda\overline{\distortion}(\overline{f}_{\lambda}).
\end{equation*}
Therefore, $\errsam (A \overline{f}_{\lambda})) \leq \errsam (A \overline{f}_t)$.
Using this, we have 
\begin{align*}
    \err (A \overline{f}_{\lambda}) - \err (A \overline{f}_t) & = (\err (A \overline{f}_{\lambda}) - \errsam (A \overline{f}_{\lambda})) + (\errsam (A \overline{f}_{\lambda})) - \errsam (A \overline{f}_t))\\&\hspace{3cm} + {(\errsam (A \overline{f}_t) - \err (A \overline{f}_t))} \\
    & \leq  {2} \left(2 \rho \Rad_S (\ApproxClass) + 3 \sqrt{\frac{\ln (\frac{8}{\delta})}{2m}}\right).
\end{align*}
with probability at least $1-\delta$, by standard Rademacher bounds applied to the class $\ApproxClass$ twice. 
\end{proof}

We should note that Propositions \ref{prop:lambda} and \ref{prop:overline} require that $\lambda$ is specified before seeing the data. However, we can use SRM to allow an exploration of a countable number of different values for this parameter before making this choice for a small additional error term. Specifically, take a sequence of candidate values $\{\lambda_k\}_{k\in\N}$ weighted by $\{w_k\}_{k\in\N}$ with $\sum_{k\in\N}w_k\le 1$. Then the same bounds hold for all $\lambda_k$, where $k\in\N$, simultaneously at the expense of an additional term of $3\sqrt{\frac{\log(1/w_k)}{2m}}$.

\section{Rademacher complexity of the class of sensitivities}\label{diffClass}

The generalisation bounds of Section \ref{sec:2} that include estimated values of the sensitivity, rely on the Rademacher complexity of the class of sensitivities $\Rad_S (\distortion \Hil)$. Arguably, this set can be at least as large as the original function class $\Hil$ in the worst case, so one may wonder whether the bounds are actually useful. In this section we look at this quantity more closely. 

Indeed, using the basic properties of the Rademacher complexity gives 
\begin{equation} \label{eqn:Crude bound 1}
    \Rad_{S} (\distortion \Hil) \leq \Rad_{S} (\Hil) + \Rad_{S} (\ApproxClass). 
\end{equation}
Moreover, this bound is tight, since equality holds when the approximating class $\ApproxClass$ is a singleton -- however, the use of a singleton $\ApproxClass$ is quite contrived, and far from what approximate algorithms are designed for.

For a fixed (possibly unlabelled) sample $S$, the set of interest in this section is the restriction of $\distortion\Hil$ to $S$,
\[
    \distortion \Hil \vert_{S} \coloneqq \left\{ \begin{pmatrix} \vert   f(x_1) - Af (x_1) \vert   \\ \vdots \\ \vert   f(x_m) - Af (x_m) \vert   \end{pmatrix} \colon f \in \Hil \right\}
\]
As before, we denoted $R_p \coloneqq \sup_{f \in \Hil}\distortionSam^p (f)$, the worst sensitivity in the chosen $p$-norm on the sample $S$. 
Note that from Assumption \ref{ass:bounded sensitivity} we have $R_p \leq C$ for all $p>0$.
We shall also use the shorthand
\[
    \x_k=\x(x_k)=\vert  f(x_k)-Af(x_k)\vert   \text{ and } \x=(\x_k)_{k\in[m]}.
\]
Note that $\distortion \Hil \vert_{S} \subseteq B_p(0,m^{1/p}R_p)$ for all $p\ge 1$,
where $\B_p(c,r)$ denotes the $p$-ball centered at $c$ with radius $r$. 

We start by putting a crude magnitude bound on $\Rad_{S} (\distortion \Hil)$, which holds irrespective of the choices of $\Hil$ and $\ApproxClass$ and is tight up to a constant factor.
The following proposition shows that, whenever we have a good approximation on the sample for all predictors in $\Hil$, the Rademacher complexity of the sensitivity class must be small in magnitude, and this bound is also tight up to a constant factor, for all choices of $p\ge 1$.
This magnitude bound will not imply a decay as $m$ increases, as we make no assumptions beyond an i.i.d. sample at this point. However this magnitude bound will be a useful reference in our later subsections, and it can also be taken in conjunction with other bounds, since one can always take the minimum of all upper bounds. 

\begin{prop}[Crude magnitude bound] \label{prop:Crude2}
For any $p\ge 1$, we have
\[
    \Rad_S (\distortion \Hil) \leq R_p
\]
Moreover, given $p$ as chosen above, suppose that $\distortion \Hil \vert_S$ nearly fills the $p$-ball of radius $R_p m^{1/p}$, in the sense that the convex hull of $\distortion \Hil \vert_S$ contains the {$p$-ball of radius $\frac{m^{1/p}}{2} R_p$ intersected with the positive orthant.} 
Then there exists a constant $C_p>0$  that only depends on the choice of $p$-norm, such that 
\[
    \Rad_S (\distortion \Hil)\geq C_p\cdot R_p.
\]
\end{prop}

\begin{proof}
By H\"older's inequality,   
\begin{align*}
    \Rad_{S} (\distortion \Hil) & = \frac{1}{m} \E_{\sigma} \sup_{f \in \Hil} \sum_{k=1}^{m} \sigma_k \vert   f(x_k) - Af (x_k) \vert   \\
    & \leq \frac{1}{m} \sup_{f \in \Hil} \sum_{k=1}^{m} \vert   f(x_k) - Af (x_k) \vert   \\
    & \leq \sup_{f \in \Hil} \left( \frac{1}{m} \sum_{k=1}^{m} \vert   f(x_k) - Af (x_k) \vert  ^p \right)^{\frac{1}{p}},\\
    & = \sup_{f \in \Hil} \distortionSam^{p}(f) = R_p
\end{align*} 
for all $p \in [1, \infty)$. 
This proves the upper bound.

We denote by $K_+$ the positive orthant, and 
let $\B_p^+\left(0,{\frac{m^{1/p}}{2} R_p}\right):= K_+ \cap \B_p \left(0,{\frac{m^{1/p}}{2} R_p}\right)$. To prove the lower bound, we recall Moreau's decomposition theorem \cite{Moreau} (see also \cite[Sec. 2.1 \& Sec. 3.1.5]{WeiWainwright}), which is the following:
Given a closed convex cone $K\subset \R^m$, denote its polar cone by $K^*=\{u\in\R^m : \langle u,u' \rangle\leq 0\text{~for all~} u'\in K\}$. Then, every vector $v\in\R^m$ can be decomposed as
\begin{align} 
v=\Pi_{K}(v)+\Pi_{K^*}(v)\text{~such that~} \langle \Pi_{K}(v),\Pi_{K^*}(v)\rangle = 0, \label{Moreau}
\end{align}
where $\Pi_{K}(u) \coloneqq \argmin_{u'\in K}\|  u-u' \|  _2$ is the orthogonal projection of $u$ into $K$. Hence we have
\begin{align}
    \Rad_S (\distortion \Hil) & = \frac{1}{m} \E_{\sigma} \sup_{f \in \Hil} \sum_{k=1}^m \sigma_k \vert   f(x_m) - Af (x_m) \vert   \\
    & = \frac{1}{m} \E_{\sigma} \sup_{\x \in \text{conv} (\distortion \Hil \vert_S)} \sum_{k=1}^m \sigma_k \x_k \\
    &\geq \frac{1}{m} \E_{\sigma} \sup_{\x \in \B_p^+\left(0,\frac{m^{1/p}}{2} R_p\right)} \sum_{k=1}^m \sigma_k \x_k \\
        &= \frac{1}{m} \E_{\sigma} \sup_{\x \in \B_p^+\left(0,\frac{m^{1/p}}{2} R_p\right)} \x^T(\Pi_{K_+}(\sigma)+\Pi_{K_+^*}(\sigma))  \label{proj} \\
    &=\frac{1}{m}\cdot\frac{m^{1/p}}{2}R_p\cdot\E_{\sigma} \|  \Pi_{K_+} \sigma\|  _{p'} \label{eq20inWeiWainwright}\\
    &=\frac{1}{m}\cdot\frac{m^{1/p}}{2}R_p \cdot \left(\frac{m}{2}\right)^{1/p'} \\ 
    &=m^{1/p+1/p'-1}\cdot 2^{-1-1/p} \cdot R_p =  \frac{R_p}{2\sqrt[p]{2}}.
\end{align}
where $p'$ is the H\"older conjugate of $p$, i.e. $1/p+1/p'=1$. In line \eqref{proj} we applied \eqref{Moreau} to $\sigma$, and \eqref{eq20inWeiWainwright} follows from the fact that $\x$ is in the positive orthant $K_+$ so $\langle u,\Pi_{K_+^*}(\sigma)\rangle \leq 0$ 
and due to the supremum equality is attained when $\x$ is a nonnegative scalar multiple of $\Pi_{K_+}(\sigma)$ -- in which case $\langle u,\Pi_{K_+^*(\sigma)}\rangle=0$. This completes the proof of the lower bound.
\end{proof}

The lower bound highlights the fact that one cannot tighten the complexity bound by more than a constant factor without making extra assumptions. In addition, we also see that non-negativity of the elements of $\distortion\Hil$ only affect this constant. 
Therefore in the next few sections we set out to find and exploit other structures in order to gain more transparency and insight on the effective magnitude of this quantity in some natural settings. Specifically, we shall discuss examples of some non-restrictive structural models from which one can read off benign conditions that give better bounds on $\Rad_S (\distortion \Hil)$. A lower magnitude of this complexity implies a smaller unlabelled data set size requirement for accurate estimation of the sensitivity, and in the case of our bounds in Sections \ref{noBalcan} and \ref{ImplementableAlgo} this may even permit solving the learning problem without the need of an additional unlabelled sample.

\subsection{Exploiting structural models of the sensitivity set}

Throughout this section we make no assumption about neither the function class $\Hil$ nor the approximating class $\ApproxClass$. So the results of this section are equally relevant to very rich classes like deep neural networks, all the way to very restricted ones like linear classes. We also make no assumption about the form of the approximating function, and indeed the approximating class is not required to be of the same architectural type as the original class.

We demonstrate the benign effects of some structural traits that the set $\distortion \Hil$  may naturally exhibit regardless of the linear or nonlinear nature of the actual predictors. Such benign structures will manifest themselves by explaining a reduced complexity $\Rad(\distortion \Hil)$ -- which in turn allow the bounds of Section \ref{sec:2} to provide a better understanding of what makes some instances of a learning problem easier than others.

Our strategy in the next subsections will be to study the complexity of the set $\distortion \Hil$ restricted to the sample $S$ (as it appears in the empirical Rademacher bounds presented in Section \ref{sec:2}) by inscribing it into various parametrised geometric shapes. These include natural structures such as the points of $\distortion\Hil_{\vert  S}$ being near-sparse, or exhibiting clusters, or having some structured sparsity type model. For this we will not actually impose any extra conditions, instead our strategy is to use these constructs to reveal how the Rademacher complexity depends on the parameters of these models. In other words, our bounds will always hold with some parameter values, as in the worst case we just recover the crude magnitude bound in Proposition  \ref{prop:Crude2}, while at the same time the effects of parameters convey more insight.

\subsubsection{Near-sparse sensitivity set}

A very natural situation is when some points in $S$ have little effect on the sensitivity of the approximation, or in other words the approximation has little effect on the predictions of part of the points of $S$. For instance in classification, points that are far from the boundary will often have the approximating function $Af$  predict in agreement with the original $f$. 

A simple way to model this situation is by having the vectors in $\distortion \Hil \vert_S$ lie near the axes corresponding to the points in $S$ which are less affected by the approximation, such as taking a shape of an axis-aligned ellipse in some Minkowski norm,
defined as 
\begin{align}
    \mathcal{E}_p(\mu) \coloneqq \left\{ x \in \R^{m} \colon \sum_{k=1}^{m} \frac{\vert  \x_k\vert  ^p}{\mu_k^p} \leq 1 \right\},\label{eq:ellipse}
\end{align}
for $p\ge 1$, where $\mu \coloneqq (\mu_1 , \dots , \mu_m) \in (0,\infty)^{m}$ are the semi-axes of the ellipse.

Note, this model is not restrictive, since we have $\distortion \Hil \vert_{S} \subset B_{p}^{m} (0,R_p m^{1/p})$, therefore $\mu_k \leq R_p m^{1/p}$ for all $k \in [m]$. However, the added flexibility of this model allows us to infer the effect of the magnitudes of the semi-axes, yielding some simple and natural conditions that improve on the worst-case magnitude guarantee in Proposition \ref{prop:Crude2}. 

The following lemma gives the exact expression for the Rademacher complexity of an ellipse in any $p$-norm. 

\begin{lem}\label{ellipse}
Let $\mu\in (0,\infty)^m$ and $p\ge 1$, and consider $\mathcal{E}_p(\mu)$ as defined in \eqref{eq:ellipse}. 
Then,
\[
    \RRad (\mathcal{E}_p(\mu)) = \frac{\|  \mu\|  _{\frac{p}{p-1}}}{m}.
\]
\end{lem}

\begin{proof}
Using H\"older's inequality, $\sigma_k \in \{ -1 , 1\}$, and the definition of $\mathcal{E}_p(\mu)$ we have
\begin{align}
  \RRad_{S} (\mathcal{E}_p(\mu)) 
    & = \frac{1}{m} \E_{\sigma} \sup_{\x \in \mathcal{E}_p(\mu)} \sum_{k=1}^{m} \sigma_k \x_k \\
    & = \frac{1}{m} \E_{\sigma} \sup_{\x \in \mathcal{E}_p(\mu)} \sum_{k=1}^{m} (\sigma_k \mu_k) \frac{\x_k}{\mu_k} \\
    & = \frac{1}{m} \E_{\sigma} \sup_{\x \in \mathcal{E}_p(\mu)} \left( \sum_{k=1}^{m} \vert  \sigma_k \mu_k\vert  ^{p'} \right)^{\frac{1}{p'}} \left( \sum_{k=1}^{m} \frac{\vert  \x_k\vert  ^p}{\mu_k^p} \right)^{\frac{1}{p}} \label{sup1}\\
    & = \frac{1}{m} \left( \sum_{k=1}^{m} \mu_k^{p'} \right)^{\frac{1}{p'}}. \label{sup2}
\end{align}
where $p'$ is the H\"older conjugate of $p$, i.e. $1/p+1/p'=1$.
The equalities \eqref{sup1} and \eqref{sup2} hold due to the supremum.
This completes the proof.
\end{proof}

For more intuition, consider the case when $p=2$, which corresponds to the usual Euclidean norm ellipse, and we can relate the RHS of the bound in Lemma \ref{ellipse} to the volume of the ellipse. 
Indeed, using the relation between the arithmetic and geometric mean,  
\[ \frac{1}{m}\|  \mu\|  _2= \frac{1}{\sqrt{m}}
 \left( \frac{1}{m} \sum_{k=1}^{m} \mu_k^2 \right)^{\frac{1}{2}} \geq \frac{1}{\sqrt{m}}\left( \prod_{k=1}^m \mu_k \right)^{\frac{1}{m}} = C_m \text{Vol} (\mathcal{E}_2(\mu))^{\frac{1}{m}},
\]
where $C_m>0$ is a constant depending only on $m$. Hence, for a fixed sample size $m$, if the quadratic mean of the $\mu_k$'s is small then the ellipse has a small volume. 

If $\distortion\Hil \subseteq \mathcal{E}_p(\mu)$
then in the worst case, $\mu_k=R_p m^{1/p}$ for all $k\in[m]$, and so
\[
    \frac{\|  \mu\|  _{\frac{p}{p-1}}}{m} \leq
    \frac{1}{m}\left(\sum_{k\in[m]} (R_p m^{1/p})^{\frac{p}{p-1}}\right)^{\frac{p-1}{p}} = R_p.
\]
Hence it is clear that the bound in Lemma \ref{ellipse} recovers the bound in Proposition \ref{prop:Crude2} in the worst case.
Thus, if $\distortion\Hil \subseteq \mathcal{E}_p(\mu)$, then Lemma \ref{ellipse} is already an improvement on Proposition \ref{prop:Crude2}.

As a model of the sensitivity set, an ellipse with high excentricity posits that most sensitivity vectors reside in a linear subspace of $\R^m$. Interesting to note that this has no implication on the form of the predictors. Indeed, even with highly nonlinear predictors (nonlinear classification boundaries for example), the fraction of points for which the predictions are distorted under the action of approximation may be expected to be small. 

However, it would be unrealistic to expect that for all functions of $\Hil$ the approximation will change the prediction for the same points and will leave alone the same points. Hence, instead of assuming that  $\distortion\Hil$ is contained in a single ellipse, for a more realistic model, we consider a union of multiple axes-aligned ellipses that cover $\distortion \Hil \vert_S$. This allows the set of points for which predictions are relatively unaffected by the approximation of some $f\in\Hil$ be different for all $f\in\Hil$.

The following proposition shows that in this model the Rademacher complexity of $\distortion \Hil \vert_S$ is bounded by the Rademacher complexity of the largest ellipse from the union and, remarkably, it does not depend on the number of ellipses in the union -- we can have countably many in this model, so the diversity of sensitivity profiles of the predictors of $\Hil$ in the span of the sample is accounted for at no expense. The vector of axis lengths for the $i$-th ellipse will be denoted by $\mu_i$. We refer to individual components of this vector by adding a second index.   

\begin{prop}[Complexity of near-sparse sensitivity set] \label{prop:Union of ellipses centred at the origin}
Let $S \subseteq \mathcal{X}$ be an i.i.d. unlabeled sample drawn from $D_x$, of size $m$. Let $l\in\mathbb{N}\cup\infty$, 
suppose that 
there exist $\mu_i\in(0,\infty)^m$ with $\mu_{i,k}\le R_pm^{1/p}$, and $\distortion \Hil \vert_{S} \subset \bigcup_{i=1}^l \mathcal{E}_p(\mu_i)$ for ellipses $\mathcal{E}_p(\mu_i)$.
Then we have the following bound
\[
    \Rad_{S} (\distortion \Hil) \leq \frac{1}{m} \max_i \|  \mu_i\|  _{\frac{p}{p-1}}.
\]
\end{prop}
The proof makes use of similar steps as the proof of Lemma \ref{ellipse}, but it does not apply the result of Lemma \ref{ellipse} directly, as it turns out that a direct approach yields the exact Rademacher complexity of the union of axis-aligned ellipses. 

\begin{proof}[Proof of Proposition \ref{prop:Union of ellipses centred at the origin}]
As $\distortion \Hil \vert_{S} \subset \bigcup_{i=1}^l \mathcal{E}_p(\mu_i)$, then using the fact that for two bounded sets $A$ and $B$ we have $\sup (A \cup B) = \max \{ \sup A , \sup B \}$, taking absolute value, the H\"older inequality, $\sigma_k \in \{ -1 , 1\}$, and the definition of $\mathcal{E}_p(\mu_i)$ gives
\begin{align}
    \Rad_{S} (\distortion \Hil) & \leq \RRad (\bigcup_{i=1}^l \mathcal{E}_p(\mu_i)) \\
    & = \frac{1}{m} \E_{\sigma} \sup_{\x \in \bigcup_{i=1}^l \mathcal{E}_p(\mu_i)} \sum_{k=1}^{m} \sigma_k \x_k \\
    & = \frac{1}{m} \E_{\sigma} \max_i \sup_{\x \in \mathcal{E}_p(\mu_i)} \sum_{k=1}^{m} \sigma_k \x_k \\
    & = \frac{1}{m} \max_i \sup_{\x \in \mathcal{E}_p(\mu_i)} \sum_{k=1}^{m} \vert  \x_k\vert   \label{symmetry}\\
    & = \frac{1}{m} \max_i \sup_{\x \in \mathcal{E}_p(\mu_i)} \sum_{k=1}^{m} \mu_{i,k} \frac{\vert  \x_k\vert  }{\mu_{i,k}} \\
    & = \frac{1}{m} \max_i \sup_{\x \in \mathcal{E}_p(\mu_i)} \left( \sum_{k=1}^{m} \mu_{i,k}^{p'} \right)^{\frac{1}{p'}} \left( \sum_{k=1}^{m} \frac{\vert  \x_k\vert  ^p}{\mu_{i,k}^p} \right)^{\frac{1}{p}} \label{Holder w eq}\\
    & = \frac{1}{m} \max_i \left( \sum_{k=1}^{m} \mu_{i,k}^{p'} \right)^{\frac{1}{p'}},
\end{align}
where $p'$ is the H\"older conjugate of $p$. The equality in \eqref{symmetry} is due to the symmetry of the set $\bigcup_{i=1}^l \mathcal{E}_p(\mu_i)$ around around each axis, and in \eqref{Holder w eq} H\"older' inequality holds with equality due to the supremum.
\end{proof}

It may be interesting to note that the model of a union of axis-aligned ellipses has an intuitive meaning of near-sparsity of sensitivities. This may also be interpreted as a kind-of near-compression bound, since Proposition \ref{prop:Union of ellipses centred at the origin} tells us that the fewer points that are affected by the approximation, the tighter the guarantee that the sensitivity estimates are accurate, hence the better the generalisation bound.

However, beyond the motivation of this intuitive meaning, our structural modelling approach has the potential to reveal additional benign conditions that might be harder to find by intuition alone. We shall modify Proposition \ref{prop:Union of ellipses centred at the origin} to get an upper bound for a union of non-axis aligned union of ellipses. As long as the ellipses share the same center (for instance, at the origin), the upper bound will still be independent of the number of ellipses in the union. 

To see this, in addition to the axis-length parameters, for each ellipse in the union take a rotation matrix $V_i\in\R^{m\times m}$ where $V_i^TV_i=V_iV_i^T=I_m$ for $i\in[l]$. The columns of $V_i$ are the principal directions for the $i$-th ellipse. We will refer to the $k$-th column of $V_i$ by $(V_i)_k$, and $(V_{i})_{k,k'}$ will denote its $(k,k')$-th element.
The $i$-th ellipse is then defined as
\begin{align}
    \mathcal{E}_p^{V_i}(\mu_i) \coloneqq \left\{ \x \in \R^{m} \colon \sum_{k=1}^{m} \frac{\vert  (V_{i})_k^T\x\vert  ^p}{\mu_{i,k}^p} \leq 1 \right\},\label{eq:ellipse2}  
\end{align}
By a change of variables, we have that $u\in  \mathcal{E}_p^{V_i}(\mu_i)$ is equivalent to $V_i^T\x\in \mathcal{E}_p(\mu_i)$. Let $\Lambda_i$ be the diagonal matrix with elements $\mu_{i,k}  \in (0,\infty)$ for $k\in[m]$, so $\Lambda_i^{-1}V_i^T\x \in \mathcal{B}_p(0,1)$. 

We no longer have symmetry around the axes, so \eqref{symmetry} becomes an inequality, and we have
\begin{align}
       \RRad (\bigcup_{i=1}^l \mathcal{E}_p^{V_i}(\mu_i))&=
    \frac{1}{m} \E_{\sigma} \max_{i\in[l]}\sup_{\x \in \mathcal{E}_p^{V_i}(\mu_i)} \sum_{k=1}^{m} \sigma_k \x_k \\
    & \le \frac{1}{m} \max_i \sup_{\x \in \mathcal{E}^{V_i}_p(\mu_i)} \sum_{k=1}^{m} \vert  \x_k\vert   \\
    & = \frac{1}{m} \max_i \sup_{\x \in \mathcal{E}^{V_i}_p(\mu_i)} \|  \x\|  _1 \\
    & = \frac{1}{m} \max_i \sup_{\x \in \mathcal{E}^{V_i}_p(\mu_i)} \|  (V_i\Lambda_i)(\Lambda_i^{-1}V_i^T\x)\|  _1 \\
    & = \frac{1}{m} \max_i \sup_{\Lambda_i^{-1}V_i^T\x 
    \in \mathcal{B}_p(0,I_m)} \|  (V_i\Lambda_i) (\Lambda_i^{-1}V_i^T\x)
    \|  _1 \label{changeOfVar}\\
    &= \frac{1}{m} \max_i  \|  V_i\Lambda_i\|  _{p\rightarrow 1}. \label{inducedNorm}
\end{align}
Eq. \eqref{changeOfVar} used the assumption that $\Lambda_i$ and $V_i$ are full rank square matrices. The last line \eqref{inducedNorm} holds by the definition of  $\|  \cdot\|  _{p\rightarrow 1}$, called the operator norm (or induced matrix norm) with domain $p$ and co-domain $1$. Such norms can only be computed explicitly in a few special cases. In particular, 

1) Whenever $V_i=I_m$ then $\|  V_i\Lambda_i\|  _{p\rightarrow 1}=\|  \mu_i\|  _{\frac{p}{p-1}}$, since $\Lambda_i$ is the diagonal matrix with elements $\mu_{i,k}$. This recovers precisely the axis-aligned setting. 

2) With $p=1$, the expression of the induced norm is known, 
$\|  V_i\Lambda_i\|  _{1\rightarrow 1}=\max_{k\in[m]} \mu_{i,k} \|   (V_i)_k\|  _1$.


We see the non-axis alignment has led to somewhat less intuitive expressions, but nevertheless the main quantity that governs the Rademacher complexity remains some notion of the size of largest ellipse. To interpret this in the context of interest here, it is enough if the sensitivities mainly reside in linear subspaces of $\distortion\Hil\vert  _{S}\subset \R^m$ for the Rademacher complexity of $\distortion\Hil$ to be small. 
Equivalently, for the estimation of sensitivities to require less unlabelled points. In other words, what we found in this analysis is that, it is a benign to have the approximation of each function in $\Hil$ mainly affect just a few \emph{linear combinations} of the sensitivities of the sample points of $S$ (not necessarily on individual points).


\subsubsection{Clustered sensitivity set}
In this section we consider another natural structure, namely when the elements of $\distortion \Hil \vert_S$ form clusters. A cluster is a subset of $\Hil$ with similar sensitivity profile on the sample $S$. We can model each cluster with a $p$-norm ellipse, each having its own center as the following
\[
    \mathcal{E}_p(c_i,\mu_i,V_i) \coloneqq \left\{ \x \in \R^m : \sum_{k=1}^m \frac{\vert  (V_i)_{k}^T(\x - c_{i})\vert  ^p}{\mu_{i,k}^p} \leq 1 \right\}.
\]
The components of the vector $\mu_i$ are the semi-axes, and the vector $c_i$ is the center of the $i$-th cluster.
This model is again un-restrictive, as there exist worst case parameter values $(c_i=0,\mu_i=R_pm^{1/p},V_i=I_m)_{i\in[l]}$ that recover the ball $\mathcal{B}_p(0,R_pm^{1/p})$ used previously in the crude bound Proposition \ref{prop:Crude2}. 

The following proposition shows that in this model,  $\Rad_S(\distortion \Hil)$ is bounded by the Rademacher complexity of the largest cluster plus an additive term that grows logarithmically with the number of clusters and linearly with the largest displacement of a cluster from the origin.
\begin{prop}[Complexity of clustered sensitivity set]\label{clusters}
Let $S \subset \mathcal{X}$ be an unlabeled sample of size $m$ drawn i.i.d. from $D_x$. 
Let $l\in\mathbb{N}$, suppose that 
there exist $\mu_i\in(0,\infty)^m, c_i\in\R^m$ and $V_i$ such that $\mu_{i,k}\le R_pm^{1/p}$, $V_i^TV_i=V_iV_i^T=I_m$ and $\distortion \Hil \vert_{S} \subseteq \bigcup_{i=1}^l \mathcal{E}_p(c_i,\mu_i,V_i)$ for $p$-ellipses. 
Then,
\[
    \Rad_S (\distortion \Hil) \leq \frac{1}{m} \max_i  \|  V_i\Lambda_i\|  _{p\rightarrow 1} + \max_i \{\|   c_i\|  _2\} \frac{ \sqrt{2 \ln l}}{m}.
\]
\end{prop}
 We see this model highlights a trade-off about the effect of large sensitivities: If a cluster only contains functions for which the approximation makes large sensitivities, the first term of the bound can still be small, but a penalty is incurred in the second term if not all function fit in the same cluster.

\begin{proof}
Let $c \colon \distortion\Hil \to \{ c_1 , \ldots , c_l \}$ be defined as 
the function that sends $\x \in \distortion\Hil$ 
to its nearest center of an ellipse, 
$c(\x):=\argmin_{c_i : i\in[l]} \sum_{k\in[m]}(V_i)_k^T(\x-c_{i})/\mu_k$. 
Ties are broken arbitrarily.

Now, adding and subtracting $c(\x_k)$ and noting that, by construction, $\{c(\x) : \x \in 
\distortion\Hil \} = \{ c_1 , \ldots , c_l \}$ 
we have
\begin{align}
    \Rad_{S} (\distortion \Hil) & \leq \RRad (\bigcup_{i=1}^l \mathcal{E}_p (c_i,\mu_i,V_i)) \\
    & = \frac{1}{m} \E_{\sigma} \sup_{\x \in \bigcup_{i=1}^l \mathcal{E}_p(c_i,\mu_i,V_i)} \sum_{k=1}^{m} \sigma_k \x_k \\
    & = \frac{1}{m} \E_{\sigma} \max_{i\in[l]} \sup_{\x \in \mathcal{E}_p(c_i,\mu_i,V_i)} \sum_{k=1}^{m} \sigma_k \x_k \\
    & \leq \frac{1}{m} \E_{\sigma} \max_{i\in[l]} \sup_{\x \in \mathcal{E}_p(c_i,\mu_i,V_i)} \sum_{k=1}^{m} \sigma_k (V_i)_k^T(\x - c_{i})
      \nonumber\\
    & \hspace{3cm} + \frac{1}{m} \E_{\sigma}  \max_{i\in[l]} \sup_{\x \in \mathcal{E}_p (c_i,\mu_i,V_i)} \sum_{k=1}^{m} \sigma_k (V_i)_k^T c_{i} \\
    & =\frac{1}{m} \E_{\sigma} \max_{i\in[l]} \sup_{V_{i}^T(\x-c_{i}) \in \mathcal{E}_p(0,\mu_i)} \sum_{k=1}^{m} \sigma_k(V_i)_k^T(\x - c_{i})
    \nonumber\\
    &\hspace{3cm} + \frac{1}{m} \E_{\sigma} \max_{i\in[l]} \sum_{k=1}^{m} \sigma_k (V_i)_k^T c_{i}.
\end{align}
We proceed by bounding the above two terms separately. 

We bound the first term by applying Proposition \ref{prop:Union of ellipses centred at the origin}, or its extension, eq. \eqref{inducedNorm}.
To bound the second term, we use Massart's lemma to get
\[
    \frac{1}{m} \E_{\sigma} \max_{i\in[l]} \sum_{k=1}^{m} \sigma_k (V_i)_k^Tc_i \leq \max_{i\in[l]} \{\|   c_i\|  _2\} \frac{ \sqrt{2 \ln l}}{m}.
\]
since $V_i$ is a rotation matrix, so $\|  V_i^Tc_i\|  _2=\|  c_i\|  $.
Combining the two bounds together completes the proof.
\end{proof}

This bound is similar in flavour to that of the complexity of a union given in \cite[Lemma 7.4]{Golowich2020} in the sense that there is a logarithmic price to pay for the number of clusters. However, by contrast, here we have an explicit constant in the second term with clear relation to the position of the ellipses, and our bound reduces to that from Proposition \ref{prop:Union of ellipses centred at the origin} if all $c_i=0$ for all $i \in [l]$.
Therefore, the above bound gives more information as to what helps decrease the Rademacher complexity. More specifically, the benign structures identified are: small number of clusters, cluster centers close to the origin, and highly concentrated (low volume) clusters.

\subsection{Exploiting the structural form of predictors}

Our analysis so far was completely independent of the specification of $\Hil$ and $\Hil_A$, and applies to any PAC-learnable hypothesis class. 
From the crude bound in \eqref{eqn:Crude bound 1} we know that a low complexity $\Hil$ always implies a low complexity $\distortion\Hil$. In this section we give a worked example of how this effect plays out in the case of hypothesis classes that are linear in the parameters. Linear models represent a well-weathered object of study at the foundation of machine prediction \cite{Vapnik}, whose high-dimensional / low sample size version has been of much interest for the puzzle of over-parameterisation, see e.g. \cite{Bartlett19}. These models also allow for nonlinearity effortlessly through a feature map.

Let $\mathbb{H}$ be a reproducing kernel Hilbert space with reproducing kernel $k \colon \X \times \X \to \R$ associated with the feature map $\Phi \colon \X \to \mathbb{H}$, so for any $x_1,x_2\in\X$, we have $k(x_1,x_2)=\langle \Phi(x_1),\Phi(x_2) \rangle_{\mathbb{H}}$.
Then our hypothesis class is  
\[
    \Hil \coloneqq \{ x \mapsto \langle w , \Phi(x) \rangle_{\mathbb{H}} : w \in \mathbb{H} \}   
\]
The familiar Euclidean space setting corresponds to $\Phi$ being the identity map and $\mathbb{H}=\R^d$.

We define our approximation operator to be $A \colon \Hil \to \ApproxClass$ defined by $A f_w (x) = \langle Q(w) , \Phi (x) \rangle_{\mathbb{H}}$ where $f_w (x) = \langle w , \Phi (x) \rangle_{\mathbb{H}}$ and $Q \colon \mathbb{H} \to \mathbb{H}$ is some approximation of the weights $w$ of the predictor $f_w$.

\begin{prop}\label{prop:lin}
Let $m \in \N$ and $S = \{ x_1 , \ldots , x_m \} \subset \X$. Then we have the following bound
\begin{align}
    \Rad_S (\distortion \Hil) \leq \frac{\sup_{w \in \mathbb{H}} \|  w - Q(w)\|  _{\mathbb{H}}}{\sqrt{m}} 
    \sqrt{\sum_{k=1}^mk(x_k,x_k)}. \label{lin}
\end{align}
\end{prop}

This is of course upper bounded by the sum of familiar bounds for linear classes $\Hil$ and $\ApproxClass$ by the triangle inequality, as already implied indeed by the crude bound \eqref{eqn:Crude bound 1}; however, the important observation from the special-case analysis of Proposition \ref{prop:lin} is that \eqref{lin} does not explicitly depend on the norm of the weight vectors, but instead it only depends on how the approximation $A$ (through $Q$) distorts the weights. In other words, we do not need bounded norms $\|  w\|  _{\mathbb H}$ for $\Rad_S (\distortion \Hil)$ to be bounded as long as the weight sensitivity $\|  w - Q(w)\|  _{\mathbb{H}}$ is bounded for the chosen operator $Q$. 

Therefore the finding we conclude from Proposition \ref{prop:lin} is that, in the generalised-linear model class considered,  {small weight-sensitivity is sufficient for dimension-independent learning} when the approximating class $\ApproxClass$ has dimension-free complexity (e.g. a constant VC dimension). This is in contrast with existing dimension-free bounds that required bounded norm, since smallness of weight sensitivity implies smallness of weight norms, but not the other way around. 

We have not found an analogous property for other hypothesis classes, and it remains an open question as to whether analyses of the sensitivity class tailored to specific classes would unearth additional insights. 

\begin{proof}[Proof of Proposition \ref{prop:lin}]
Since $\sigma_k$ are uniform on $\{-1,1\}$, we can remove the absolute value, and by the linearity of inner products, and the Cauchy-Schwarz inequality we have 
\begin{align*}
    \Rad_S (\distortion \Hil) & = \frac{1}{m} \E_{\sigma} \sup_{f \in \Hil} \sum_{k=1}^m \sigma_k \vert   f (x_k) - A f (x_k) \vert   \\
    & {=}
    \frac{1}{m} \E_{\sigma} \sup_{w \in \mathbb{H}} \sum_{k=1}^m \sigma_k ( \langle w , \Phi (x_k) \rangle_{\mathbb{H}} - \langle Q(w) , \Phi (x_k) \rangle_{\mathbb{H}} ) \\
    & = \frac{1}{m} \E_{\sigma} \sup_{w \in \mathbb{H}} \langle w - Q(w) , \sum_{k=1}^m \sigma_k \Phi (x_k) \rangle_{\mathbb{H}} \\
    & \leq \frac{1}{m} \sup_{w \in \mathbb{H}} \|  w-Q(w)\|  _{\mathbb{H}} \E_{\sigma} \left\|   \sum_{k=1}^m \sigma_k \Phi (x_k) \right\|  _{\mathbb{H}}.
\end{align*}
Finally it is know that \cite[Theorem 6.12]{MRT2018}
\[
    \E_{\sigma} \left\|   \sum_{k=1}^m \sigma_k \Phi (x_k) \right\|  _{\mathbb{H}} \leq \left[ \sum_{k = 1}^m k(x_k,x_k) \right]^{\frac{1}{2}}.
\]
This completes the proof.
\end{proof}

\section{Discussion and potential extensions}\label{extensions}
This section outlines some ways in which our results can be straightforwardly extended.

\subsection{Faster convergence rates for approximable hypothesis classes}


We already commented that the smaller the sensitivity threshold $t$ the tighter the bounds. 
Next we show that a uniformly small $t$, with approximation sensitivity specified with $p=2$, can even speed up the convergence rate. 


\begin{prop}[Sensitivity estimation bound for uniformly approximable classes]\label{alt-Lemma2.1}
Let $S \subset \X^m$ be a sample drawn i.i.d. from the marginal distrabution $D_x$ of size $m$. 
Suppose there exist $t>0$ such that $\distortion^{2} (f) \leq t$ for all $f \in \Hil$ then with probability at least $1-\delta$ we have 
\begin{align}
    \sup_{f \in \Hil} \vert  \distortion^{1} (f) - \distortionSam^{1} (f) \vert   \leq 6 \Rad_S (\distortion \Hil) + t \sqrt{\frac{2 \ln (\frac{1}{\delta})}{m}} + \frac{6 C \ln (\frac{1}{\delta})}{m}.
    \label{variance_bound}
\end{align}
\end{prop}

\begin{proof}
First note that from assumption \ref{ass:bounded sensitivity} we have $\| f - Af\|_{\infty} < C$ and
we have the following bound on the variance of the function $f - A f$,
\begin{align*}
    \Var_X [f(X) - A f(X)] &= \E_X [(f(X) - Af (X))^2] - \E_X [f(X) - A f(X))]^2\\ 
    &\leq \E_X [(f(X) - A f (X))^2]\\
    &\leq (\distortion^{2} (f))^2 = t^2,
\end{align*}
where the last inequality uses Jensen's inequality.
The result then follows from \cite[Theorem 2.1]{bartlett2005} by setting $\alpha =\frac{1}{2}$.
\end{proof}


Proposition \ref{alt-Lemma2.1} bounds the deviation between the true sensitivity and its sample estimate in terms of the global sensitivity threshold $t$ of functions in $\Hil$. Whenever $t$ is sufficiently small, then the last term will dominate the $t$-dependent term, which in turn decays with $m$ at a faster rate. 

The observation that the sensitivity threshold $t$ acts as a variance to control the rate could also be further refined using localisation to replace the global sensitivity threshold with the sensitivities of individual functions and relax the requirement that the entire class $\Hil$ is well approximable, at the expense of a more involved machinery of local Rademacher complexities \cite{bartlett2005}, which we do not pursue here, and which would likely need a specialised treatment to bound the local complexity for particular choices of $\Hil$. 

However, it may be interesting to highlight that even in the simple global analysis of Proposition \ref{alt-Lemma2.1}, together with the findings of Section \ref{diffClass}, we can readily extract some fast rate conditions, as the following:
\begin{itemize}
\item If the sensitivity set $\distortion\Hil_{\vert  S}$, is near-sparse, so that $\distortion \Hil \vert_{S} \subset \bigcup_{i=1} \mathcal{E}_p(\mu_i)$, for possibly countable number of ellipses, where $\mathcal{E}_p(\mu_i)$ is the $i$-th ellipse having axis lengths concatenated into the vector $\mu_i$, then whenever there exists a constant $\kappa>0$ independently of $m$ such that $\max_{i\in[l]}{\|  \mu\|  _2}<\kappa$, then the rate of convergence for sensitivity estimation becomes of order $1/m$ up to log factors by Proposition \ref{prop:Union of ellipses centred at the origin}. 
\item If $\distortion\Hil_{\vert  S}\subset \bigcup_{i=1}^l \mathcal{E}_p(c_i,\mu_i,V_i)$ can be covered by a union of $l\in\mathbb{N}$ elliptic clusters where $\mathcal{E}_p(c_i,\mu_i,V_i)$ is the $i$-th ellipse centered at $c_i$, with axis lengths $\mu_i$ and orientation $V_i$, and there exist constants $\kappa_1,\kappa_2>0$ independent of $m$ such that $\max_i  \|  V_i\Lambda_i\|  _{2\rightarrow 1} < \kappa_1$ and $\max_i \{\|   c_i\|  _2\} \le \kappa_2$, then by Proposition \ref{clusters} the rate becomes of order $1/m$ up to log factors.
\item By the crude magnitude bound in Proposition \ref{prop:Crude2}, if $t$ is negligible then $\Rad_S (\distortion \Hil) \leq R_2$ is also negligible, so the rate of convergence for estimation of sensitivities becomes essentially of order $1/m$ at the expense of a negligible additive term.
\end{itemize}

\subsection{Stochastic approximation schemes}
The approximation schemes assumed so far were deterministic. Many approximation schemes are in fact stochastic in nature, therefore, in this section we discuss how to straightforwardly adapt our framework to stochastic approximation schemes. 

Let $(\Omega, \mathcal{F},\mathbb{P})$ be a probability space.
Then we define a Stochastic approximation scheme by $A \colon \Omega \times \Hil \to \Hil_{\A}$, where $\Hil_{\A} \coloneqq \{ A_{\omega} f : \omega \in \Omega \text{ and } f \in \Hil \}$. 
Then for a fixed $\omega \in \Omega$ we have an approximation operator $A_{\omega} \colon \Hil \to \Hil_{\omega}$ where $\Hil_{\omega} \coloneqq \{ A_{\omega} f : f \in \Hil \}$; that is, for a fixed $\omega$ we have one approximation operator.
Thus, when $\vert  \Omega\vert   = 1$ we reduce to the deterministic setting.
Also, for a fixed $f \in \Hil$ we have the collection of possible approximations to $f$ the set$\{ A_{\omega} f : \omega \in \Omega \}$.

Now we define $\distortion[\omega] (f) \coloneqq \distortion[A_{\omega}] (f)$, and then for a fixed arbitrary $\omega \in \Omega$ we have with probability at least $1-\delta$, that
\begin{align}
    \err (f) \leq \errsam (A_{\omega} f) + \rho \distortion[\omega] (f) + 2 \rho \Rad_S (\Hil_{\omega}) + 3 \sqrt{\frac{\ln (\frac{2}{\delta})}{2m}}, \label{stoch0}
\end{align}
for all $f \in \Hil$.
This uniform bound follows directly from Lemma \ref{lem:error to aproximate error} combined with a standard Rademacher bound, and for fixed $\omega$ the first two terms on its RHS correspond to the objective function of the Algorithm \eqref{AlgoWithKnownsensitivity} in Section \ref{noBalcan}.

We can make this independent of a particular random instance, e.g. by considering expectation. Although we cannot take expectation on both sides as this would incur a union bound over infinitely many sets, we can simply write
\begin{align*}
    \err (f) & = \err (f) - \E_{\omega} \err (A_{\omega} f) + \E_{\omega} \err (A_{\omega} f) \\
    & \leq \rho \E_{\omega} \distortion[\omega] (f) + \E_{\omega} \err (A_{\omega} f) - \E_{\omega} \errsam (A_{\omega} f) + \E_{\omega} \errsam (A_{\omega} f) \\
    & \leq \rho \E_{\omega} \distortion[\omega] (f) + \sup_{f \in \Hil} \left[ \E_{\omega} \err (A_{\omega} f) - \E_{\omega} \errsam (A_{\omega} f) \right] + \E_{\omega} \errsam (A_{\omega} f).
\end{align*}
Now applying Jensen's inequality, we have 
\[
\sup_{f \in \Hil} \left[ \E_{\omega} \err (A_{\omega} f) - \E_{\omega} \errsam (A_{\omega} f) \right] \le 
\E_{\omega} \sup_{f \in \Hil} \left[  \err (A_{\omega} f) - \errsam (A_{\omega} f) \right]
\]
and the argument of the expectation can be bounded in terms of the Rademacher complexity $\mathcal{R}_m (\Hil_{\omega})$. Thus, we have the following uniform bound expressed in terms of the expected sensitivity, the expected Rademacher complexity of the small approximating class, and a new empirical error term that, due to the expectation may be interpreted as a data augmentation loss.  
That is, we have, with probability at least $1-\delta$, the following
\begin{align}
    \err (f) \leq \E_{\omega} \errsam (A_{\omega} f) + \rho \E_{\omega} \distortion[\omega] (f) + 2 \rho \E_{\omega} \mathcal{R}_m (\Hil_{\omega}) + \sqrt{\frac{\ln (\frac{1}{\delta})}{2m}}.\label{stochBound}
\end{align}
Minimising the first two terms on its RHS could be used to justify a regularised data augmentation algorithm in analogy with our previous algorithm in \eqref{AlgoWithKnownsensitivity}. 

Likewise, one can introduce estimates of the expected distortion $\distortion[\omega] (f)$ from unlabeled data. Alternatively, if the approximation operator $A$ satisfies a variance condition, namely     
that $\left[ \E_{\omega} \|  A_w f - f \|  _{L^2}^2 \right]^{\frac{1}{2}} \leq \alpha \mathcal{C} (f)$ for all $f \in \Hil$, where $\mathcal{C} (f)$ is some property of $f \in \Hil$, 
then we have, 
by Jensen's inequality and the variance condition,
 $   \E_{\omega} [\distortion[\omega] (f)] \leq \E_{\omega} [\distortion[\omega] (f)^2]^{\frac{1}{2}} = \left[ \E_{\omega} \E_{x \sim D_x} \vert  A_{\omega} f(x) - f(x)\vert  ^2 \right]^{\frac{1}{2}} = \left[ \E_{\omega} \|  A_{\omega} f - f\|  _{L^2}^2 \right]^{\frac{1}{2}} \leq \alpha \mathcal{C} (f)$. So we see this variance condition on $A$ provides another instance where need for additional unlabelled data is eliminated in the case of stochastic approximation operators. A similar condition, formulated on the level of parameters, is frequently encountered in the literature of quantisation for learning and optimisation, such as in stochastic rounding \cite{Alistarh,Wen}.    

\section{Conclusions}\label{end}
We end our study with a high-level summary. Inspired by the recent surge of interest in model compression and approximate learning algorithms in the context of small device settings, we studied the role of approximability in generalisation, both in the full precision and in the approximated settings. 
Our main findings can be summarised as follows: (1) For any given PAC-learnable problem, and any approximation scheme, target concepts that have low sensitivity to the approximation are learnable from a smaller labelled sample, provided sufficient unlabelled data. This is achieved by using approximation to modify the loss function and isolating a sensitivity term in the generalisation error. The modified loss function has a lower complexity in comparison with the original, pushing the complexity of the learning problem onto the class of sensitivity functions -- which in turn only requires unlabeled data for estimation whenever the original loss is Lipschitz. (2) Our analysis yielded algorithms showing that it is possible to learn a good predictor whose approximation has same generalisation guarantee as the full precision predictor. Owing to the generality of our approach, such provably accurate approximate predictors can be used with a variety of model compression and approximation schemes, and potentially deployed in memory-constrained settings.
(3) Our algorithms use unlabelled data to estimate the sensitivity of predictors to the given approximation operator,  and this needs not be disjoint from the labelled training set. Moreover, while the required 
unlabelled sample complexity can be large in general, we highlighted several examples of natural structure in the class of sensitivities that significantly reduce, and possibly even eliminate, the need of additional unlabelled data. At the same time, structural properties of the sensitivity class shed new light onto the question of what makes certain instances of learning problems easier than others.

In this work we built on the classic Rademacher complexity framework which was well suited to support our alternating between uniform generalisation bounds and associated learning algorithms.  
In future work it would be interesting to study these questions in other learning theory frameworks such as PAC-Bayes, and perhaps even non-uniform frameworks. 

\section*{Acknowledgments}
The work of both authors was funded by the EPSRC Fellowship EP/P004245/1 “FORGING: Fortuitous Geometries
and Compressive Learning”.

\bibliographystyle{plain} 
\bibliography{Qbib.bib}

\begin{thebibliography}{10}

\bibitem{Alistarh}
Dan Alistarh, Demjan Grubic, Jerry~Z. Li, Ryota Tomioka, and Milan Vojnovic.
\newblock Qsgd: Communication-efficient sgd via gradient quantization and
  encoding.
\newblock In {\em Proceedings of the 31st International Conference on Neural
  Information Processing Systems (NIPS'17)}, Curran Associates Inc., page
  1707–1718, 2017.

\bibitem{Arora2018}
Sanjeev Arora, Rong Ge, Behnam Neyshabur, and Yi~Zhang.
\newblock Stronger generalization bounds for deep nets via a compression
  approach.
\newblock In {\em International Conference on Machine Learning}, pages
  254--263. PMLR, 2018.

\bibitem{Ashbrock2020}
Jonathan Ashbrock and Alexander~M. Powell.
\newblock Stochastic {M}arkov gradient descent and training low-bit neural
  networks.
\newblock {\em Sampl. Theory Signal Process. Data Anal.}, 19(15), 2021.

\bibitem{bartlett2005}
Peter~L. Bartlett, Olivier Bousquet, and Shahar Mendelson.
\newblock Local {R}ademacher complexities.
\newblock {\em The Annals of Statistics}, 33(4):1497--1537, 2005.

\bibitem{Bartlett19}
Peter~L. Bartlett, Philip~M. Long, G{\'a}bor Lugosi, and Alexander Tsigler.
\newblock Benign overfitting in linear regression.
\newblock {\em Proceedings of the National Academy of Sciences},
  117(48):30063--30070, 2020.

\bibitem{Bartlett2002}
Peter~L. Bartlett and Shahar Mendelson.
\newblock Rademacher and {G}aussian complexities: Risk bounds and structural
  results.
\newblock {\em Journal of Machine Learning Research}, 3:463--482, Nov 2002.

\bibitem{Baykal2018}
Cenk Baykal, Lucas Liebenwein, Igor Gilitschenski, Dan Feldman, and Daniela
  Rus.
\newblock Data-dependent coresets for compressing neural networks with
  applications to generalization bounds.
\newblock In {\em 7th International Conference on Learning Representations
  (ICLR)}, 2019.

\bibitem{Bu21}
Yuheng Bu, Weihao Gao, Shaofeng Zou, and Venugopal~V. Veeravalli.
\newblock Population risk improvement with model compression: An
  information-theoretic approach.
\newblock {\em Entropy (Basel)}, 23(10), 2021.

\bibitem{Balcan2010}
Maria-Florina B\v{a}lcan and Avrim Blum.
\newblock A discriminative model for semi-supervised learning.
\newblock {\em Journal of the ACM}, 57(3), 2010.

\bibitem{Chapelle}
Olivier Chapelle, Bernhard Sch\"{o}lkopf, and Alexander Zien.
\newblock {\em Semi-Supervised Learning (Adaptive Computation and Machine
  Learning)}.
\newblock The MIT Press. 2006.

\bibitem{Cheng2020}
Y.~Cheng, D.~Wang, P.~Zhou, and T.~Zhang.
\newblock A survey of model compression and acceleration for deep neural
  networks. arxiv.
\newblock preprint, 2017.

\bibitem{Choudhary2020}
Tejalal Choudhary, Vipul~Kumar Mishra, Anurag Goswami, and Sarangapani
  Jagannathan.
\newblock A comprehensive survey on model compression and acceleration.
\newblock {\em Artificial Intelligence Review}, pages 1--43, 2020.

\bibitem{Courbariaux2015}
Matthieu Courbariaux, Yoshua Bengio, and Jean-Pierre David.
\newblock Binaryconnect: Training deep neural networks with binary weights
  during propagations.
\newblock In {\em In Advances in neural information processing systems}, pages
  3123--3131, 2015.

\bibitem{Denil2013}
Misha Denil, Babak Shakibi, Laurent Dinh, Marc’Aurelio Ranzato, and Nando
  de~Freitas.
\newblock Predicting parameters in deep learning.
\newblock In {\em In Advances on Neural Information Processing Systems}, 2013.

\bibitem{Denton2014}
Emily~L Denton, Wojciech Zaremba, Joan Bruna, Yann LeCun, and Rob Fergus.
\newblock Exploiting linear structure within convolutional networks for
  efficient evaluation.
\newblock In {\em In Advances in neural information processing systems}, pages
  1269--1277, 2014.

\bibitem{Gao19}
Weihao Gao, Yu-Han Liu, Chong Wang, and Sewoong Oh.
\newblock Rate distortion for model {C}ompression:{F}rom theory to practice.
\newblock In Kamalika Chaudhuri and Ruslan Salakhutdinov, editors, {\em
  Proceedings of the 36th International Conference on Machine Learning},
  volume~97 of {\em Proceedings of Machine Learning Research}, pages
  2102--2111, 09--15 Jun 2019.

\bibitem{Golowich2020}
Noah Golowich, Alexander Rakhlin, and Ohad Shamir.
\newblock Size-independent sample complexity of neural networks.
\newblock {\em Information and Inference: A Journal of the IMA}, 9(2):473--504,
  2020.

\bibitem{Han2015}
Song Han, Huizi Mao, and William~J. Dally.
\newblock Deep compression: Compressing deep neural network with pruning,
  trained quantization and {H}uffman coding.
\newblock In Yoshua Bengio and Yann LeCun, editors, {\em 4th International
  Conference on Learning Representations, (ICLR)}, 2016.

\bibitem{Hubara2017}
Itay Hubara, Matthieu Courbariaux, Daniel Soudry, Ran El-Yaniv, and Yoshua
  Bengio.
\newblock Quantized neural networks: Training neural networks with low
  precision weights and activations.
\newblock {\em The Journal of Machine Learning Research}, 18(1):6869--6898,
  2017.

\bibitem{Menghani2021}
Gaurav Menghani.
\newblock Efficient deep learning: A survey on making deep learning models
  smaller, faster, and better.
\newblock preprint, arXiv, 2021.

\bibitem{MRT2018}
Mehryar Mohri, Afshin Rostamizadeh, and Ameet Talwalkar.
\newblock {\em Foundations of machine learning}.
\newblock MIT press. 2018.

\bibitem{Moreau}
Jean~Jacques Moreau.
\newblock Proximit\'e et dualit\'e dans un espace {H}ilbertien.
\newblock {\em Bulletin de la Soci\'et\'e Math\'ematique de France},
  93:273--299, 1965.

\bibitem{Rastegari2016}
Mohammad Rastegari, Vicente Ordonez, Joseph Redmon, and Ali Farhadi.
\newblock Xnor-net: Imagenet classification using binary convolutional neural
  networks.
\newblock In {\em European conference on computer vision}, pages 525--542.
  Springer, 2016.

\bibitem{Ravi2019}
Sujith Ravi.
\newblock Efficient on-device models using neural projections.
\newblock In Kamalika Chaudhuri and Ruslan Salakhutdinov, editors, {\em
  Proceedings of the 36th International Conference on Machine Learning},
  volume~97 of {\em Proceedings of Machine Learning Research}, pages
  5370--5379, 09--15 Jun 2019.

\bibitem{Suzuki2018}
Taiji Suzuki, Hiroshi Abe, Tomoya Murata, Shingo Horiuchi, Kotaro Ito, Tokuma
  Wachi, So~Hirai, Masatoshi Yukishima, and Tomoaki Nishimura.
\newblock Spectral pruning: Compressing deep neural networks via spectral
  analysis and its generalization error.
\newblock In {\em Proceedings of the Twenty-Ninth International Joint
  Conference on Artificial Intelligence, {IJCAI-20}}, pages 2839--2846, 2020.

\bibitem{Suzuki2020}
Taiji Suzuki, Hiroshi Abe, and Tomoaki Nishimura.
\newblock Compression based bound for non-compressed network: unified
  generalization error analysis of large compressible deep neural network.
\newblock In {\em 8th International Conference on Learning Representations
  (ICLR)}, 2020.

\bibitem{semisupSurvey}
Jesper~E. van Engelen and Holger~H. Hoos.
\newblock A survey on semi-supervised learning.
\newblock {\em Machine Learning}, 109(2):373--440, 2020.

\bibitem{Vapnik}
Vladimir Vapnik.
\newblock {\em Statistical learning theory}.
\newblock Wiley. 1998.

\bibitem{WeiWainwright}
Yuting Wei, Martin~J. Wainwright, and Adityanand Guntuboyina.
\newblock The geometry of hypothesis testing over convex cones: Generalized
  likelihood tests and minimax radii.
\newblock {\em The Annals of Statistics}, 47(2):994--1024, 2019.

\bibitem{Wen}
Wei Wen, Cong Xu, Feng Yan, Chunpeng Wu, Yandan Wang, Yiran Chen, and Hai Li.
\newblock Terngrad: Ternary gradients to reduce communication in distributed
  deep learning.
\newblock In I.~Guyon, U.~V. Luxburg, S.~Bengio, H.~Wallach, R.~Fergus,
  S.~Vishwanathan, and R.~Garnett, editors, {\em Advances in Neural Information
  Processing Systems}, volume~30 of {\em Curran Associates, Inc.}, 2017.

\bibitem{Zhou2019}
Wenda Zhou, Victor Veitch, Morgane Austern, Ryan~P Adams, and Peter Orbanz.
\newblock Non-vacuous generalization bounds at the imagenet scale: A
  {PAC}-{B}ayesian compression approach.
\newblock In {\em 7th International Conference on Learning Representations
  (ICLR)}, 2019.

\end{thebibliography}

\end{document}